%% file: paper.tex
\documentclass{article}
\usepackage{microtype}
\usepackage{graphicx}
\usepackage{booktabs} 
\usepackage{hyperref}
\usepackage{cancel}

\usepackage[preprint]{neurips_2025}

\usepackage{amsmath}
\usepackage{amssymb}
\usepackage{mathtools}
\usepackage{amsthm}
\usepackage[capitalize,noabbrev]{cleveref}
\usepackage{url}
\usepackage{graphicx}
\usepackage{caption}
\usepackage{subcaption}
\usepackage{comment}
\usepackage{wrapfig}

\theoremstyle{plain}
\newtheorem{thm}{Theorem}
\newtheorem{defn}{Definition}
\newtheorem*{defn*}{Definition}
\newtheorem{exmp}{Example} 
\newtheorem{lemma}{Lemma}

\newcommand{\vect}[1]{#1}
\usepackage{xargs}

\newcommandx{\subv}[2]{\ensuremath{{#1}_{#2}}}
\newcommandx{\supv}[2]{\ensuremath{{#1}^{(#2)}} }
\newcommandx{\subsupv}[3]{\ensuremath{{#1}^{(#2)}_{#3}} }
\newcommandx{\sumv}[2]{\ensuremath{ \sum_{#1}^{#2} } }
\newcommandx{\anglev}[2]{\ensuremath{ \left\langle #1, #2 \right\rangle } }
\newcommandx{\norm}[1]{\ensuremath{ \left\| #1 \right\|_{2}^{2} } }
\newcommandx{\dist}[2]{\ensuremath{ d^{max}(#1,#2) } }

\newcommandx{\grad}[1]{\ensuremath{ \nabla_{#1} } }
\input{supp.tex}
\newtheorem*{thm*}{Theorem}
\newtheorem*{lemma*}{Lemma}
\usepackage[disable,textsize=tiny]{todonotes}
\usepackage[textsize=tiny]{todonotes}
\title{On Understanding of the Dynamics of Model Capacity in Continual Learning}
\input{supp.tex}
\author{%
  Supriyo Chakraborty \\
  AI Foundations\\
   Capital One \\
  130 5th Avenue, New York  10011 \\
  \texttt{ supriyo.chakraborty@capitalone.com} \\
  \And
    Krishnan Raghavan \\
  Mathematics and Computer Science\\
  Argonne National Laboratory\\
  Lemont, IL-60439 \\
  \texttt{kraghavan@anl.gov} \\
  Equal contribution
}
\begin{document}
\maketitle
\begin{abstract}
    The stability-plasticity dilemma, closely related to a neural network’s (NN) capacity—its ability to represent tasks—is a fundamental challenge in continual learning (CL). Within this context, we introduce \emph{CL’s effective model capacity (CLEMC)} that characterizes the dynamic behavior of the stability-plasticity balance point. We develop a difference equation to model the evolution of the interplay between the NN, task data, and optimization procedure. We then leverage CLEMC to demonstrate that the effective capacity—and, by extension, the stability-plasticity balance point is inherently non-stationary. We show that regardless of the NN architecture or optimization method, a NN's ability to represent new tasks diminishes when incoming task distributions differ from previous ones. We conduct extensive experiments to support our theoretical findings, spanning a range of architectures—from small feedforward network (FNN) and convolutional networks (CNN) to medium-sized graph neural networks (GNN) and transformer-based large language models (LLMs) with millions of parameters. 
\end{abstract}

\input{main}
\begin{ack}
This work is supported  by the U.S. Department of Energy for the SciDAC 5 RAPIDS institute and the 
DOE Early Career Research Program award. This research used resources of the Argonne Leadership Computing Facility, 
which is a DOE Office of Science User Facility supported under Contract DE-AC02- 06CH11357.
\end{ack}

\bibliographystyle{plain}
\bibliography{reference}

\appendix
\section{Technical Appendices and Supplementary Material}
Technical appendices with additional results, figures, graphs and proofs may be submitted with the paper submission before the full submission deadline (see above), or as a separate PDF in the ZIP file below before the supplementary material deadline. There is no page limit for the technical appendices.
\section*{Supplementary Files}
We will begin by restating some preliminaries, these are the exact copy of the initial text in Section~\ref{sec:clemc} of the paper.
\section{Preliminaries}
\input{prelim.tex}

\section{Visualization of CLEMC Formulation}
The visualization of the recursive dynamic-program for CL is shown in Figure~\ref{fig:cap}. The value function and its progressive evolution is shown in black. At each task $k$, the value function considers all the previous tasks (arrows adding forgetting costs from prior intervals) and all future tasks (arrows adding forgetting costs from future intervals). The computation of the forgetting cost at task $k$, is also shown as the summation of the loss terms for each of the tasks in the interval $[1,\ldots,k]$ (green lines). Corresponding model weights optimized for tasks are also shown for completeness.
\begin{figure}
    \centering
    \includegraphics[width=\linewidth]{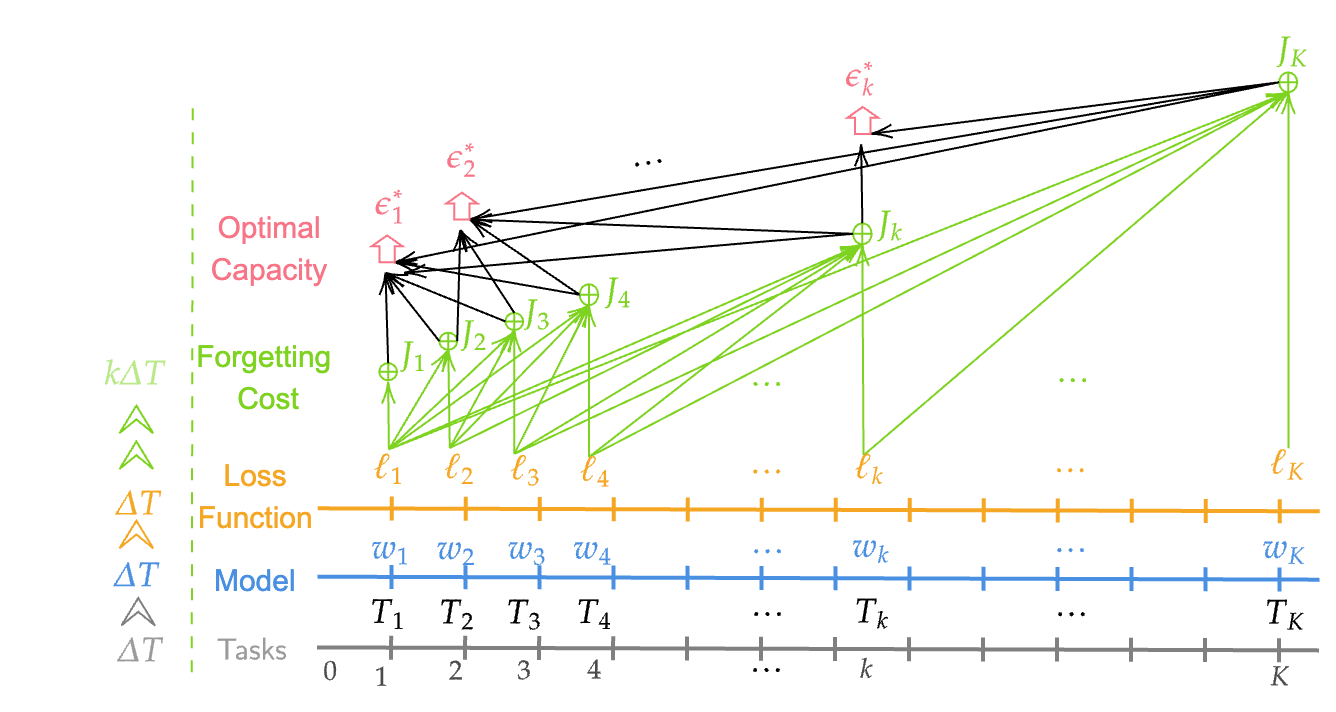}
    \caption{Visualizing the dynamic-program based CL formulation: Given a model at $w_1$ and a task a $k$, the forgetting cost, $J_k$, is the sum of loss over tasks $[1, k]$ (green arrows), therefore, $J_2 = J_1+\ell_2, $ $J_3 = J_1+ J_2 + \ell_3$ and so on.  The smallest possible forgetting loss at $k$ provides a effective capacity at k $FEMC$, the sum of these smallest possible $FEMC'$s provides the optimal capacity~(in red).}
    \label{fig:cap}
\end{figure}

\section{First Difference}
We will derive the notion of first difference in capacity as a function of the forgetting cost.

\input{lem__FD.tex}
Next, we will derive the lower bound on the first difference in capacity which stems from a lower and upper bound on capacity.

\section{Lower bound on first difference}
\input{LBFD.tex}
This lower bound then leads to the conclusion that capacity is non-stationary and diverges with increase in weight update or divergence between subsequent tasks. This non-stationarity extends to experience replay and experience replay with regularization 

\section{Divergence with respect to weights}
\input{thm_non_stati_weights.tex}
Finally we show our main result, that is, if a small change is introduced by every task, it accumulate to result in a divergent capacity.
\section{Divergence with respect to tasks}
\input{thm_non_stati_tasks.tex}

\section{Details for Case Study 4}
\input{exp_details.tex}


\section{Visualization for deeper understanding of the impact of CL on the LLM models} 

\emph{Setup:} We randomly sampled $64\times64$ parameters (2\% of the MLP parameters for 8M and 0.007\% for 134M) and tracked how their weights changed from the start ($\bf{w}_k^{0}$) to the end of training ($\bf{w}_k^{*}$) for each task $k$. We then correlated this with the capacity in (Fig.~\ref{fig:llm_data}). Note, the weight changes caused by each task correspond to the second term in \eqref{eq:capacity_lb} which is used to characterize capacity. For this example, the last checkpoint from one task serves as the starting point for the next, i.e., $\bf{w}_k^{0} = \bf{w}_{k-1}^{*}$. Although only a small sample of weights was used, repeated trials showed consistent trends.

\begin{figure}[h]
\includegraphics[width=0.5\textwidth]{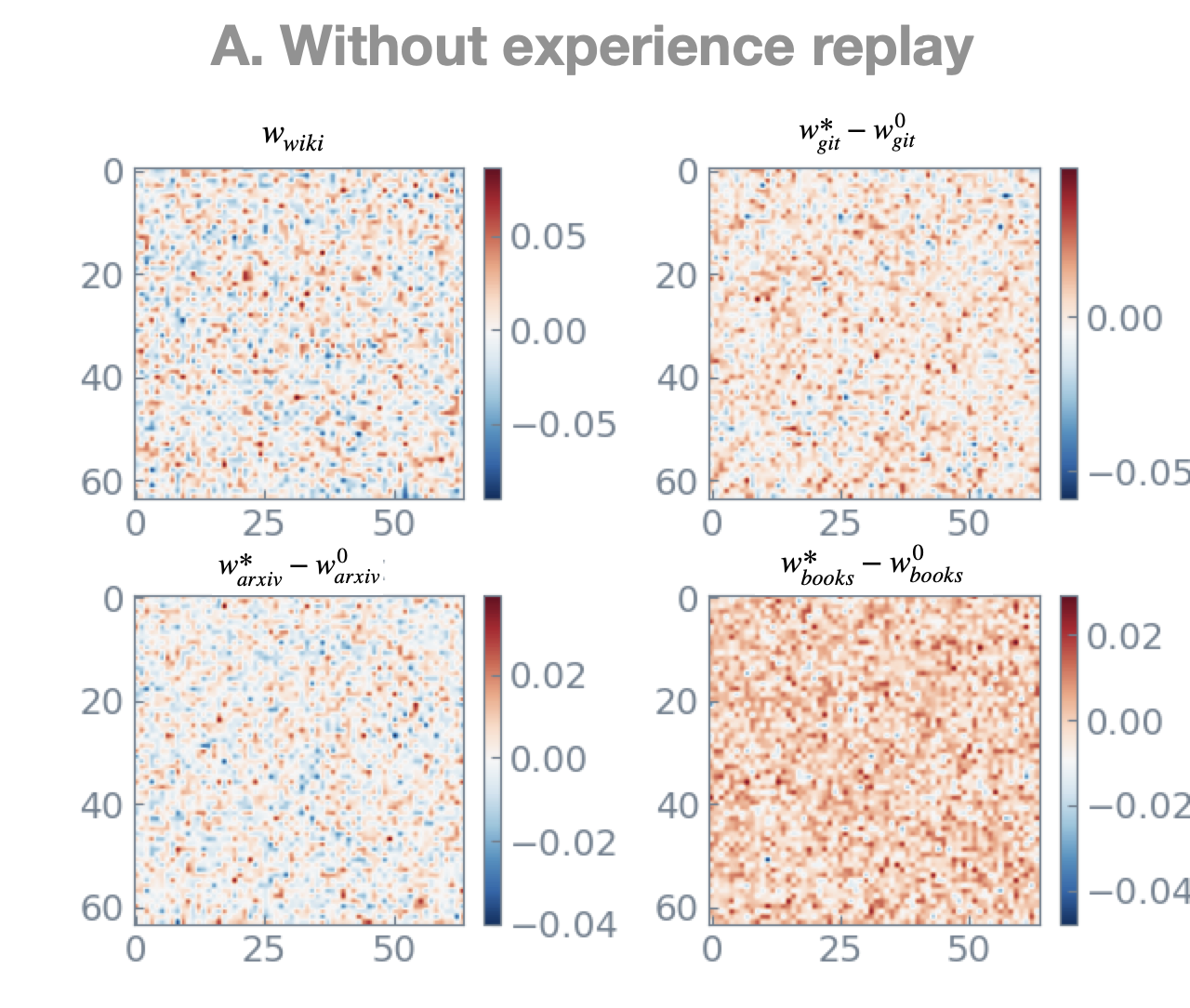}
\includegraphics[width=0.5\textwidth]{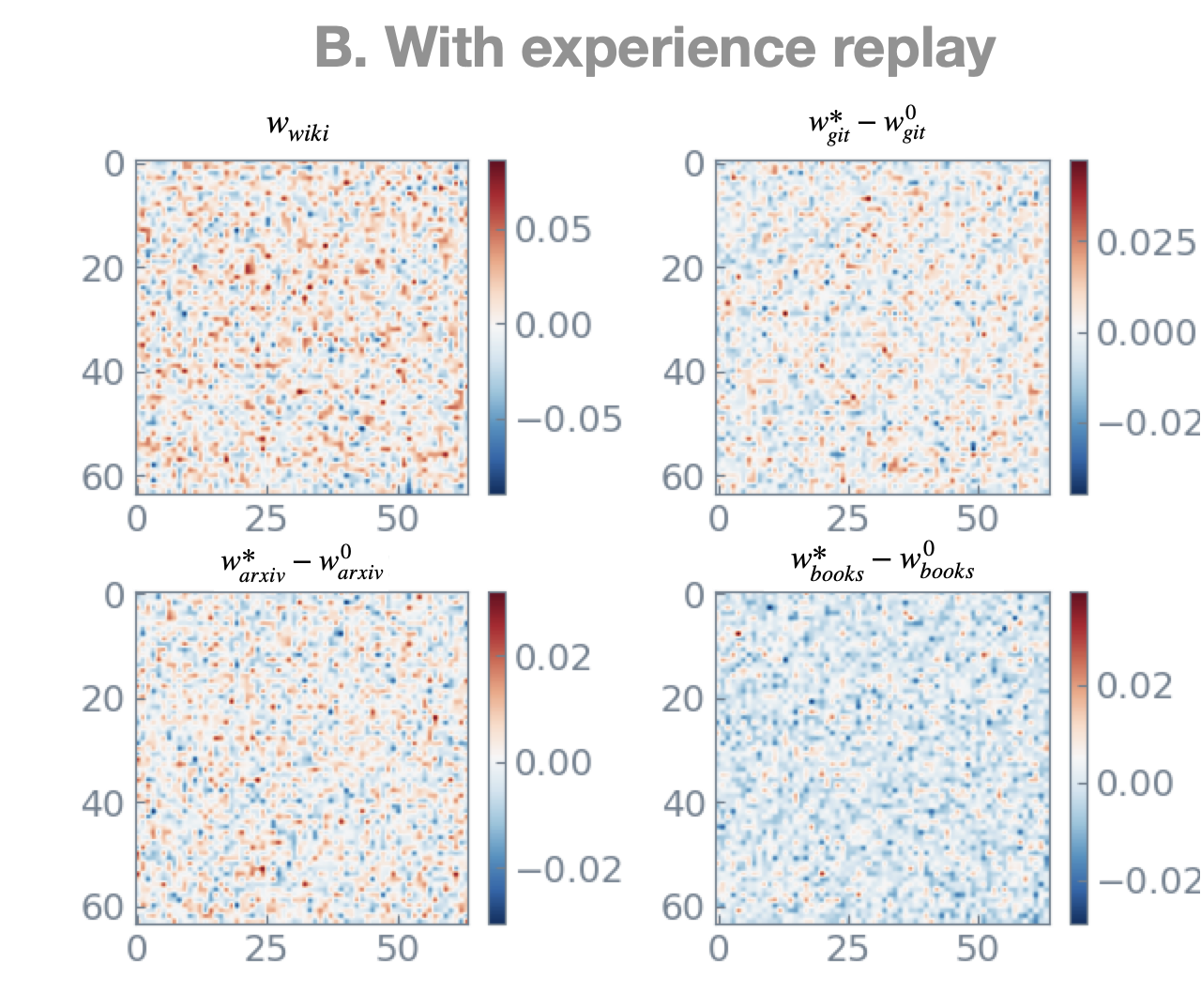}
    \caption{For a task $k$, the $64 \times 64$ heat map shows the difference in weights from the initial value, $\bf{w}_{k}^{0}$, at the start of training to the final value, $\bf{w}_{k}^{*}$, at the end of CL training. The weights are randomly sampled from the MLP sublayers in the 8M parameter model. Task arrival order: \texttt{wiki} $\rightarrow$ \texttt{git} $\rightarrow$ \texttt{arxiv} $\rightarrow$ \texttt{books}.}
    \label{fig:8m_proj}
\end{figure}

\begin{figure}
\includegraphics[width=0.5\textwidth]{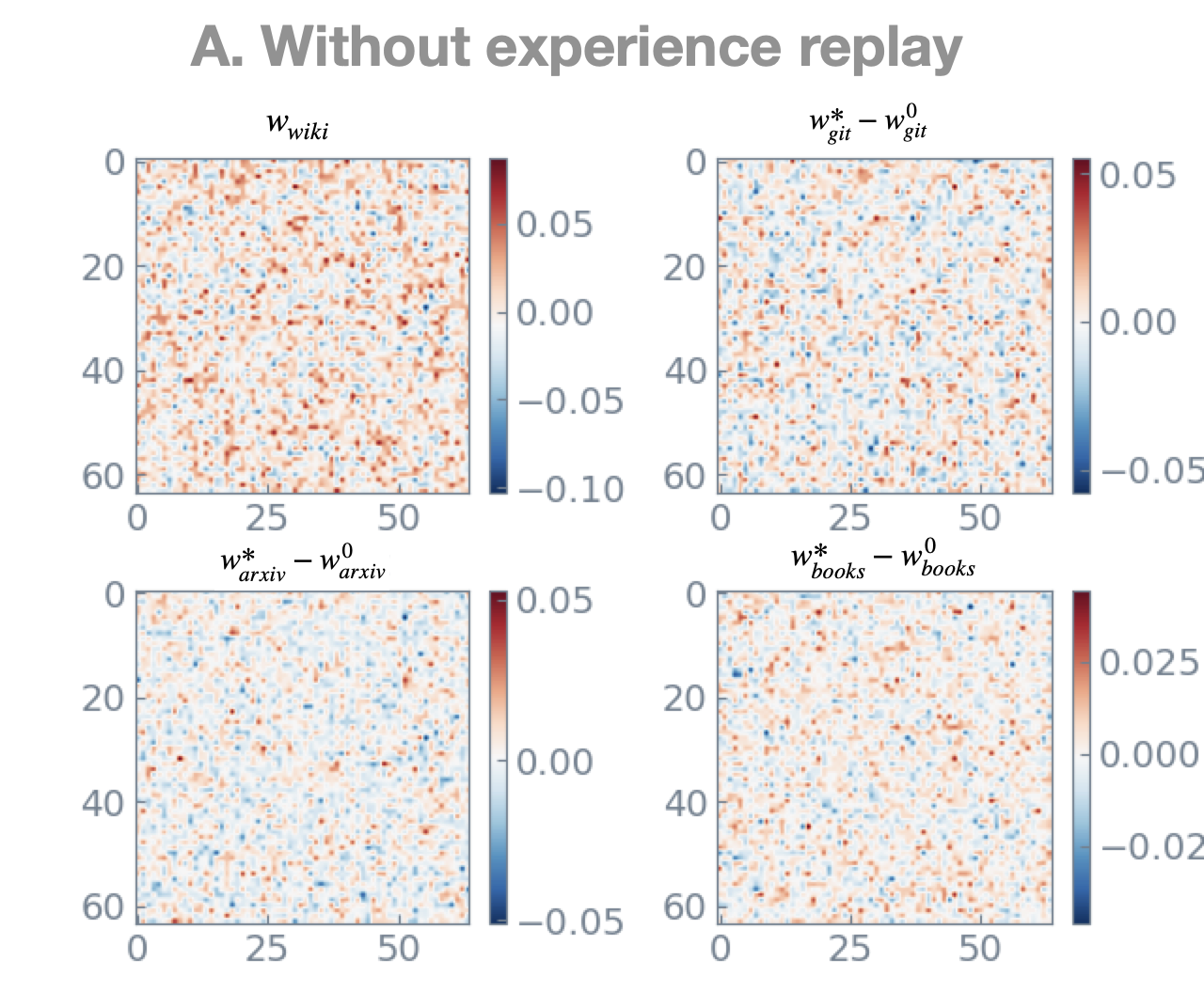}
\includegraphics[width=0.5\textwidth]{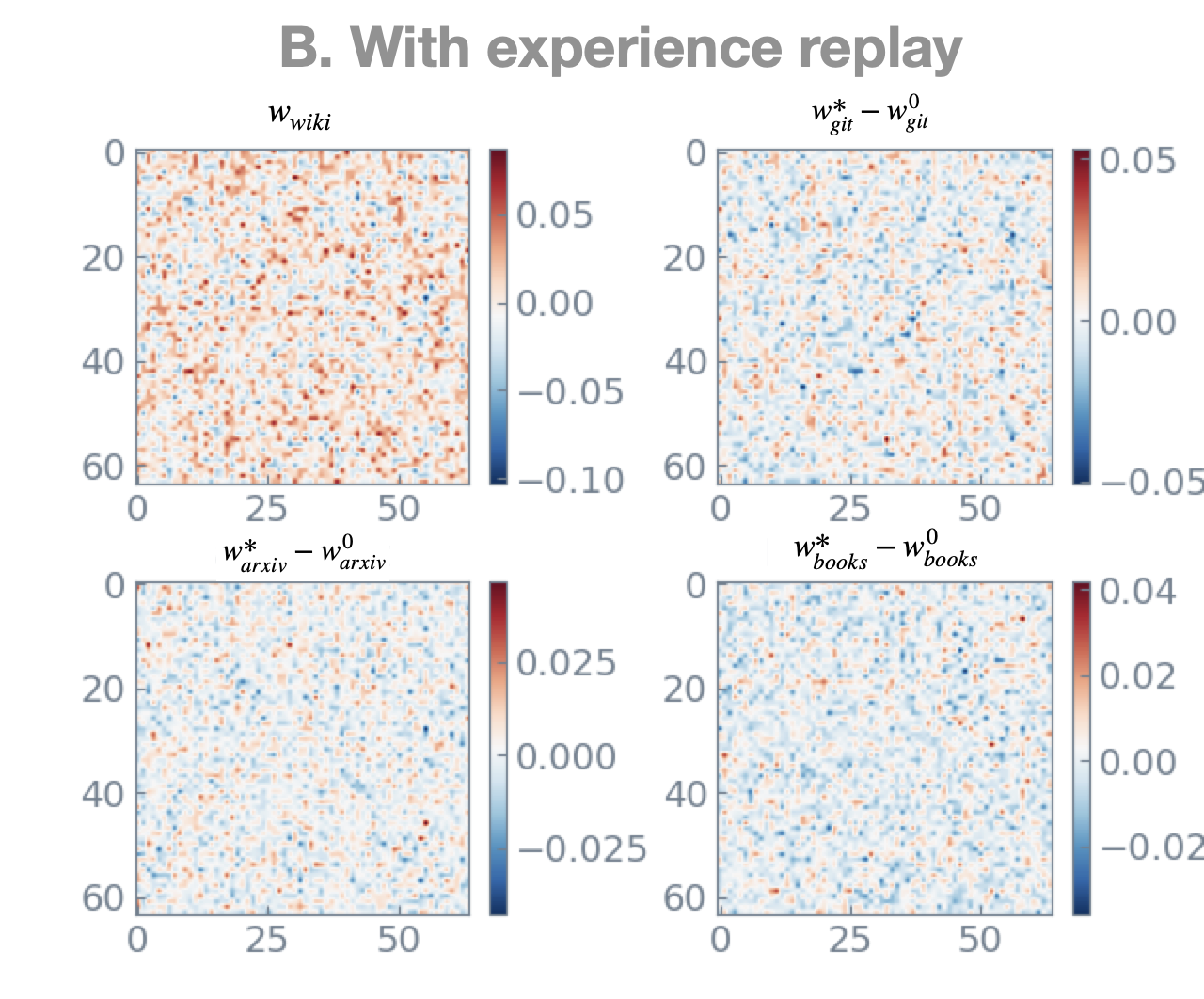}
    \caption{For a task $k$, the $64 \times 64$ heat map shows the difference in weights from the initial value, $\bf{w}_{k}^{0}$, at the start of training to the final value, $\bf{w}_{k}^{*}$, at the end of CL training. The weights are randomly sampled from the MLP sublayers in the 134M parameter model. Task arrival order: \texttt{wiki} $\rightarrow$ \texttt{git} $\rightarrow$ \texttt{arxiv} $\rightarrow$ \texttt{books}.}
    \label{fig:134m_proj}
\end{figure}

\emph{Analysis:} For the 8M model without ER, large weight changes (red in Fig.~\ref{fig:8m_proj}(A)) lead to high capacity and increased forgetting.  In the \texttt{arxiv} task, smaller changes (blue/red) show less learning and more forgetting correlating to the two terms in Lemma \ref{lem:Der_capacity} where we quantify, how weight and task changes affect the balance point.
Significant weight changes occur for the \texttt{git} task which effect the second term in Lemma \ref{lem:Der_capacity} to increase generalization (Fig. \ref{fig:8m_proj}(A)). In contrast, with ER (Fig. \ref{fig:8m_proj}(B)), weight changes between tasks are more controlled (more blue than red), reflecting how the two terms in Lemma \ref{lem:Der_capacity} balance each other. For the \texttt{books} task, weight changes are minimal (more blue), indicating marginal model adjustment, lower forgetting, and lower capacity values because the first term no longer balances the second~(as shown in Theorem \ref{thm:task_nonstationary_weights}).

For the 134M model, we observe similar trends in weight updates. Without ER~(Fig. \ref{fig:134m_proj}(A)), initial changes are slightly larger and continue to increase with each subsequent task. As with the 8M model, increased capacity and significant parameter changes indicate poor representation capability of the model. On the other hand, with ER~(Fig. \ref{fig:134m_proj}(B)), weight changes are more regularized (more blue than red) as prior tasks reduce the amount of increase in the capacity. 
\framebox{\parbox{0.9\textwidth}{
  The submitted manuscript has been created by UChicago Argonne, LLC, Operator of Argonne National Laboratory (“Argonne”).
  Argonne, a U.S. Department of Energy Office of Science laboratory, is operated under Contract No. DE-AC02-06CH11357.
  The U.S. Government retains for itself, and others acting on its behalf, a paid-up nonexclusive, irrevocable worldwide
  license in said article to reproduce, prepare derivative works, distribute copies to the public, and perform publicly
  and display publicly, by or on behalf of the Government.  The Department of Energy will provide public access to these
  results of federally sponsored research in accordance with the DOE Public Access Plan.
  \url{http://energy.gov/downloads/doe-public-access-plan}
  }}

\end{document}

%% file: main.tex
\section{Introduction}
Humans can easily adapt to multiple tasks. However, when neural networks~(NN) seek to mimic this behavior~\cite{wang2024comprehensive}, they exhibit a phenomenon known as catastrophic forgetting, where the model forgets older tasks while learning new ones~\cite{wang2024comprehensive}. This well recorded issue is seen irrespective of the NN architecture, from simple linear adaptive systems~\cite{lawrence1993parameter} to massive large language models~\cite{ramasesh2021effect, luoInvestigatingForgettingPreTrained2023}. The field of artificial intelligence that studies this phenomenon is known as continual learning~(CL).

In recent years, numerous studies in CL~\cite{dohareLossPlasticityDeep2023, wangImprovingPlasticityOnline2023, kumarMaintainingPlasticityContinual2023, chenStabilityPlasticityDilemmaContinual, raghavan2021formalizing} have shown that the core issue in CL is the trade-off between forgetting prior information~(catastrophic forgetting) and learning new information~(generalization), also known as the stability-plasticity dilemma. 
Independent lines of inquiry have examined the role of the model characteristics, optimization algorithms, and task distributions on this tradeoff. For instance, the crucial role of model overparameterization in achieving optimal performance in the CL paradigm has been studied in ~\cite{mirzadeh2022wide, goldfarb2023analysis, goldfarb2024the}. Similarly, \cite{mirzadeh2020understanding, StepsizeOptimizationContinual, trantungSharpnessGradientAware2023, zhangDensityDistributionbasedLearning2023, zhaoAdaptCLAdaptiveContinual2023} study the role of various optimization procedures (e.g., of step-size, learning rate, experience-replay, regularization), and \cite{kim2023learnability} study the learnability of CL when subsequent task distributions are overlapping. 
While these works study different aspects of the CL problem, such as model and data in \cite{lan2023elephant} or the model and optimization procedure in \cite{kim2023learnability, mirzadeh2020understanding, chenStabilityPlasticityDilemmaContinual}, they do not consider the complex interplay between the model, optimization procedures and tasks.


In this work, we provide holistic insights into this interplay which is naturally characterized by the solution to the CL problem known as the ``stability plasticity balance point''~\cite{kim2023learnability, mirzadeh2020understanding, chenStabilityPlasticityDilemmaContinual}. In developing this general framework for analyzing CL, we make two contributions. First, we extend the definition of capacity from \cite{nakkiran2021deep} to the CL paradigm, describing capacity (Def~\ref{def:CLEMC}) as the consequence of an interplay between network architecture, hyperparameters, and weights measured through the forgetting cost.  Second, with a dynamic programming-driven formulation of CL~(refer to Fig.~\ref{fig:cap} in Appendix B for a visualization of our recursive formulation), we elucidate the connection between capacity and the balance point~(Lemma~\ref{lem:Der_capacity}). We show that the networks' ability to represent tasks is directly related to capacity through the balance point, that is, higher the capacity (measured through forgetting cost) lower the representation ability of the network.
Subsequently, we lower bound the capacity~(Theorem~\ref{thm:lower_bound_capacity}) and demonstrate in Theorems~\ref{thm:task_nonstationary_weights}, and \ref{thm:task_nonstationary_tasks} that capacity, and by extension, the balance point, is non-stationary in the CL setting. 

We validate the veracity of our theoretical claims through four case studies using well-studied datasets corresponding to different problem setups and two major classes of CL methods, \textit{experience replay with and without regularization}. In the first case, we use a synthetically generated sine wave dataset~\cite{javed2019meta} and show that the capacity of a feed-forward NN~(FNN) in a regression problem diverges. Next, we leverage a standard convolutional NN~(CNN) for image classification problem through the Omniglot dataset~\cite{beaulieu2020learning} and show divergence of capacity. We then demonstrate these conclusion with a graph neural network for graph classification problems and finally, develop a detailed study using large language models~(LLM) to demonstrate our results on large models. Our findings confirm that our theoretical results hold even when we scale from a simple FNN to a 134 million parameter LLM. Code for the first three case studies is provided at \url{/r/ContLearn-30E7/}. 

\section{Related Works}
Starting from \cite{french1999catastrophic} in 1999 to \cite{goldfarb2024the} in 2024, numerous works have attempted to model/reduce catastrophic forgetting in neural networks. A simple taxonomy of recent published works reveals four categories: regularization-based~\cite{bonicelliEffectivenessEquivariantRegularization2023, kongOvercomingCatastrophicForgetting2023, nokhwalRTRARapidTraining2023, kumarMaintainingPlasticityContinual2023}, model architecture-based~\cite{auddyEffectOptimizerInitializer2023, danielsEfficientModelAdaptation2023, dejaAdaptAlignContinual2023, fanDynamicSubgraphDistillation2023, hanAdaptiveReorganizationNeural2023, lan2023elephant, yeWassersteinExpansibleVariational2023}, experience replay-based~\cite{chrysakisOnlineContinualLearning2020, harunGRASPRehearsalPolicy2023, kongTrustRegionAdaptiveFrequency2023, miAdversarialRobustMemoryBased2023, zengContinualLearningDirichlet2023} and other optimization approaches for CL efficiency~\cite{StepsizeOptimizationContinual,trantungSharpnessGradientAware2023, zhangDensityDistributionbasedLearning2023, zhaoAdaptCLAdaptiveContinual2023}. This huge body of work is focused on improving empirical performance. On the other hand, empirical attempts to study the characteristics of the CL problem have been made as well~\cite{dohareLossPlasticityDeep2023, harunOvercomingStabilityGap2023, lange2023continual, luoInvestigatingForgettingPreTrained2023}. For instance, \cite{dohareLossPlasticityDeep2023} study the loss of plasticity in CL whereas \cite{harunOvercomingStabilityGap2023, lange2023continual} study a phenomenon known as stability gap frequently observed in CL methods. 

The empirical investigative studies cover a wide range of neural network architectures as well, going from FNN/CNN in \cite{dohareLossPlasticityDeep2023, harunOvercomingStabilityGap2023, lange2023continual} to large language models in \cite{luoInvestigatingForgettingPreTrained2023,ramasesh2021effect}. Despite such a huge body of literature, there have only been a few attempts to study CL from a theoretical standpoint. The key reason behind this is that the NN learning problem in CL domain is rather complex to study requiring stringent assumptions that are scarcely held in practice. This is clearly seen from the few approaches that do theoretically analyze the problem. For instance, works in \cite{goldfarb2024the, goldfarb2023analysis, evron2023continual} study the effect of over parameterization and task similarity on forgetting with a linear model under two tasks. Catastrophic forgetting in the presence of task similarity in analyzed in the NTK regime in~\cite{Doan2020ATA}. On the other hand~\cite{lee2021continual} and \cite{lin2023theory} study the complete CL problem with a linear two layer NN. To the best of our knowledge, the only approach that does not make either a two task assumption or assume linearity of the model is \cite{kim2023learnability} but instead focuses on the class incremental setting. 

To provide a general framework to analyze CL, we take a Lyapunov analysis standpoint, a tool that has been used in the control literature~\cite{cencini2013lyapunov}.  In contrast with the existing literature, we analyze the CL problem through a dynamic programming-driven optimal control point of view following the perspective from~\cite{raghavan2021formalizing}. The only assumptions required are twice differentiability and Lipschitz continuity of the loss function- two very practical assumptions in the NN learning domain and our analysis extends to a series of tasks. In a similar vein to \cite{mirzadeh2022wide} we also perform Taylor series approximation to get this differential equation characterization, however, \emph{our theoretical analysis easily extends from a simple FNN to a llm- a very novel contribution to the CL literature}. To the best of our knowledge there has been no theoretical study, where the analysis considers a dynamical behavior of the CL problem that extends across FNN/CNN/GNN and LLM.

\section{Continual Learning Effective Model Capacity~(CLEMC)}
\label{sec:clemc}
Let $x$ and $y$ be random variables corresponding to input and output probability spaces with support $\vect{\mathcal{X}}$ and $\vect{\mathcal{Y}}$ and $\mathcal{B}(\vect{\mathcal{X}})$ and $\mathcal{B}(\vect{\mathcal{Y}})$ representing the corresponding Borel algebras. Define $\vect{t}$ as a random variable denoting the joint space of $\vect{x} \times \vect{y}$ with a model $f_{(\vect{w}, \vect{h})} : \vect{\mathcal{X}} \rightarrow \vect{\mathcal{Y}}$ being specified using weights $\vect{w}$ and hyperparameters $\vect{h}$. Given compact sets $\vect{\mathcal{W}}$ over $\vect{w}$ and $\vect{\mathcal{H}}$ over $\vect{h},$ the goal is to learn the weights by searching over the hypothesis space $\vect{f} = \{ f_{(\vect{w}, \vect{h})} , \forall \vect{h} \in \vect{\mathcal{H}}, \vect{w} \in \vect{\mathcal{W}} \}$ through a loss function $\ell_{w, h}(\vect{t}).$ 
In this paper, we will assume that the hyperparameter/architecture is fixed and therefore, will drop the notation $h$ and denote loss simply as $\ell_{w}(t)$. Throughout the paper, we will assume $x_k = x(k)$ and use them interchangeably, and define boldface $\mathbf{k} = [1,2,\cdots, k]$. In this context, we characterize the effective model capacity as follows.

\noindent \textbf{Effective Model Capacity:} We assume that $\ell_{w}(\vect{t})$ is continuous and twice differentiable over the support $\vect{\mathcal{X}} \times \vect{\mathcal{Y}}$ or $\vect{\mathcal{X}}$, and a compact set $\vect{\mathcal{W}}$. Under these assumptions, let $\ell_{min} = \mathcal{O}_{\mathcal{W}}(\vect{T}) =  min_{\vect{w} \in   \vect{\mathcal{W}}} \quad {E}_{ \vect{t} \in \vect{T} } [\ell_{w}( \vect{t})]$ be the optimization procedure with $\vect{T}$ being a dataset of samples $\vect{t}$ with $\vect{T} \subset \mathcal{B}(\vect{\mathcal{T}})$. Then, given the best hyperparameter/architecture configurations,  the optimization procedure $\mathcal{O}_{\mathcal{W}}$ seeks to find the weights $\vect{w}^* \in \vect{\mathcal{W}}$ that minimize the loss over a dataset. Given this setting, we define the effective model capacity~(the upper/lower bounds derived in the appendix) as the smallest achievable loss that remains unchanged even when additional data or training is used.
\begin{defn}[Effective Model Capacity~(EMC)]\label{def:EMC}
Given $\vect{\mathcal{W}}$ as the weight space and $\vect{T} \in \mathcal{B}(\vect{\mathcal{T}})$ with an optimization procedure $\mathcal{O}_{\vect{\mathcal{W}}}(\vect{T}),$ the EMC of the model $f$ is given as
\begin{align}\label{eq:EMC}\tag{EMC}
    \epsilon = \underset{\vect{T} \in \mathcal{B}(\vect{\mathcal{T}}) }{min}  \quad \big[ \mathcal{O}_{\vect{\mathcal{W}}}(\vect{T}) \big] =\underset{\vect{T} \in \mathcal{B}(\vect{\mathcal{T}})  }{min}  \quad \big[ \underset{ \vect{w} \in \vect{\mathcal{W}}}{ min } \quad \underset{ \vect{t} \in \vect{T} }{E} [\ell_{w}( \vect{t})] \big]
\end{align}
\end{defn}
Def \ref{eq:EMC} takes an approximation error perspective (as in~\cite{niyogi1996relationship}), however, unlike \cite{niyogi1996relationship}, the minimum achievable loss in \eqref{eq:EMC} depends on the optimization procedure used, model architecture, and the dataset. It is also similar to the capacity definition in \cite{nakkiran2021deep}, with the key distinction being that \cite{nakkiran2021deep} focuses on the number of data points that are properly represented by the model. However, this way of defining capacity is often inadequate because numerical superiority over samples alone (without considering the data distribution characteristics) doesn't ensure model usefulness \cite{fernandez2018smote}. Since, a CL problem requires careful attention to the distribution, we define capacity through the forgetting loss.


\noindent \textbf{Characterizing the CL Balance Point:}
 CL involves learning a sequence of tasks indexed by $k \in [1, K], K \in \mathbb{N}$, where task $k$ is represented by its dataset $\vect{T}(k)$. The collection of all tasks until $k$ can then be denoted as $\subv{\mathbf{T}}{k} = \{ \vect{T}(1), \vect{T}(2), \ldots, \vect{T}(k) \}$ with $\subv{\vect{\mathcal{T}}}{k}$ being the cumulative support pertaining to all the task data the model has seen until $k$. Given weight set $\subv{\vect{\mathcal{W}}}{k}$, loss function $\subv{\ell}{\subv{\vect{w}}{k}}(t), t \in \subv{\vect{\mathcal{T}}}{k}$, the model at $k$ is denoted by $f_{\subv{\vect{w}}{k}}$. The goal of CL is to maintain memory of all observed tasks, then the forgetting cost for the interval $\mathbf{k} = [1,k]$ is given as
\vspace{-2.5mm}
\begin{align}\label{eq:CL_forget}
\underset{ \subv{\vect{w}}{k} \in \subv{\vect{\mathcal{W}}}{k} }{ min } J_{\subv{\vect{w}}{k}}(\subv{ \mathbf{T}}{k}) =\underset{ \subv{\vect{w}}{k} \in \subv{\vect{\mathcal{W}}}{k} }{ min } \sum_{i =1}^{k} \gamma_i \left[ \underset{ \vect{t} \in \vect{T}(i)  }{E} \left[  \ell_{\subv{\vect{w}}{k}}( \vect{t}) \right] \right] \nonumber \;\;\; \forall \vect{T}(i) \in \subv{\mathbf{T}}{k}
\tag{$J_{F}$}
\end{align}
where, $\gamma$ ensures boundedness of $J_{\subv{\vect{w}}{k}}(\subv{ \mathbf{T}}{k})$ (see \cite{raghavan2021formalizing}, Lemma 1). The growth of forgetting cost over progressively increasing task intervals (as new tasks arrive) is shown in Fig.~\ref{fig:cap} (Appendix B, in green). The forgetting cost formulation in \eqref{eq:CL_forget} is the standard in the CL literature~\cite{mirzadeh2020understanding} but, has two key limitations~\cite{chenStabilityPlasticityDilemmaContinual, goldfarb2023analysis} that we highlight using the following illustrative example.

\begin{exmp}
    Consider three learning tasks with feasible regions $\vect{\mathcal{W}}_1, \vect{\mathcal{W}}_2$, and $\vect{\mathcal{W}}_3$, centered at ideal solutions $\vect{w}_1^*, \vect{w}_2^*$, and $\vect{w}_3^*$. The naive cost setup in \eqref{eq:CL_forget} ignores the following interactions.
    
    \noindent \textbf{Sequential Optimization:}  Solving the first task (attaining $\vect{w}_1^*$) means the second task must start from $\vect{w}_1^*$. Therefore, $\vect{w}_1^*$ and its distance from $\mathcal{W}_1 \cap \mathcal{W}_2$ (the feasible region all solutions that work on both tasks 1 and 2) determines how close we can  get to $\vect{w}_2^*.$  In general, as the optimal solution for tasks $[1, k-1]$ is used as the starting point for task $k$. the feasible region of the previous tasks has an influence on the subsequent task~\cite{evron2023continual}[Theorem 3.1]. 
    
    \noindent \textbf{Influence of future tasks:} If the second task induces a significant deviation from $\vect{w}_1^*$, large forgetting is seen~(see \cite{evron2023continual}, Figure 1). Conversely, if the new task has no influence, there's no generalization.

\end{exmp}

 It is clear with this example that each tasks' solution has an influence on the future task and at the same time, future tasks performance dictates how well the the model can do on the present tasks. That is, there is an interplay between future tasks and the present task. Mathematically, a complete CL~\cite{raghavan2021formalizing} characterization must therefore consider both the sequential optimization over tasks as well as how each tasks' solution impacts future tasks. Thus, the complete CL problem is
\vspace{-2mm}
\begin{align}\label{eq:CL_opt}
\supv{V}{*}(u_k) =  \underset{ u_k}{ min } \quad {\sum}_{i = k}^{K} \left[  J_{\subv{\vect{w}}{i}}( \subv{ \mathbf{T}}{i}  )  \right], \;\; \mbox{where} \;  u_{k} = \{ \subv{\vect{w}}{i}, \nonumber i= k,k+1, \cdots K \}
\tag{CL}
\end{align}
\vspace{2mm}
The optimization problem in \eqref{eq:CL_opt} provides the value function, where previous tasks are perfectly remembered (optimizing the sum of forgetting loss, \eqref{eq:CL_forget}) and future tasks will be perfectly learnt (for task $k$, optimizing also for $[k+1, \ldots, K]$ via successive update of model weights). That is, given a starting weight set $\vect{w}_1^* \in \mathcal{W}_1,$ the solution to the CL problem with $K$ expected tasks is  $\{\vect{w}_1^* \in \mathcal{W}_1, \vect{w}_2^* \in \mathcal{W}_1 \cap \mathcal{W}_2, \vect{w}_3^* \in \mathcal{W}_1 \cap \mathcal{W}_2 \cap \mathcal{W}_3 \cdots \vect{w}_k^* \in \cap_{k=1}^{K}\mathcal{W}_1\}$  and $\supv{V}{*}(\{ \vect{w}_1^*, \vect{w}_2^*, \vect{w}_3^*\})$ is the total cost~(corresponding to the balance point). Naturally, the value of $\ell_{min}$(see Def \ref{def:EMC}) corresponding to each of these $\vect{w}_i^*, i= 1, 2, 3, \cdots, K$ describes how well the model performs at the respective stages of the CL problem and therefore (summation of the losses) quantifies capacity in the CL setting. The value function and its progressive evolution is also illustrated in Fig. \ref{fig:cap} (Appendix B, in black). 
We now extend Def \ref{def:EMC} to define effective model capacity for a CL problem.

\noindent \textbf{CL Effective Model Capacity and Balance Point:}
For ease of exposition, we begin by stating
\begin{defn}[Forgetting Effective Model Capacity~(FEMC)]\label{def:FEMC}
For task $k \in [1, K]$, dataset $\mathbf{T}_k$, weight set $\mathcal{W}_k$, optimization procedure $\mathcal{O}_{\mathcal{W}_k }( \mathbf{T}_k)$, \ref{eq:EMC} at $k$, $\epsilon_k = min_{\mathbf{T}_{k}, \subv{\vect{w}}{k}} J_{\subv{\vect{w}}{k}}(\mathbf{T}_{k})$, we define FEMC at task $k$ as:
\vspace{-2mm}
\begin{align}\label{eq:femc}\tag{FEMC}
      \text{FEMC}(k) = \max_{\mathbf{k}} \epsilon_{\mathbf{k}} = \max \{ \epsilon_1, \epsilon_2,  \cdots, \epsilon_k \} \end{align}
\end{defn}
$FEMC(k)$ at each $k$ is defined by the highest forgetting loss in the interval $[1,k].$ For example, in a three-task scenario, the FEMC at task $3$, $FEMC(3) = \max \{\epsilon_1, \epsilon_2, \epsilon_3\}$, and is determined by the task the model forgets the most. We now define CL effective model capacity as follows.

\begin{defn}[Effective Model Capacity for CL~(CLEMC)]\label{def:CLEMC}
For a task $k \in [1, K]$, we define CLEMC as the sum of FEMC across all possible tasks as
\vspace{-2mm}
\begin{align}\label{eq:clemc}\tag{CLEMC}
      \subsupv{\epsilon}{*}{k} =\sum_{i =k}^{K} \text{FEMC}(i) = \sum_{i =k}^{K} \max_{\mathbf{i}}\; \epsilon_{\mathbf{i}}\end{align}
\end{defn}
Def~\eqref{def:CLEMC} is closely related to the forgetting loss through FEMC.
If the model learns multiple tasks, we initially obtain the FEMC corresponding to each task, and then, the  $\subsupv{\epsilon}{*}{k}$ is the sum of individual task FEMC (illustrated in Fig. \ref{fig:cap} (Appendix B, in red)). Since the individual task FEMC is proportional to the loss function, perfect representation of the underlying tasks is implied by $\subsupv{\epsilon}{*}{k}=0$ and representation (and FEMC) gets poorer and poorer as $\subsupv{\epsilon}{*}{k}$ increases. Notably, $\subsupv{\epsilon}{*}{k}$ measures the models' CL performance. 
\emph{In summary, to obtain CLEMC at task $k$, we must consider the effect of previous tasks in $[1,\ldots, k]$ on $k$ and the effect of $k$ on the future tasks  $[k, \ldots, K]$. To capture the former, FEMC is defined as the highest forgetting loss among all previous tasks and the latter is summarized by summing FEMC terms over future.}

{\bf Connecting CL and CLEMC:} Similar to \eqref{eq:clemc}, the measure of models' performance has also been defined proportional to the value of the forgetting loss. For instance, \cite{kim2023learnability}[Def 3.1] defines learnability as the gap between empirical risk and the smallest risk in the hypothesis space, but without the minimization over different data samples. Furthermore, \cite{kim2023open}[Theorem 1] suggests that necessary and sufficient conditions for good CL are proportional to effective learning on prior tasks, defined through the forgetting loss. In contrast with the above, where just loss on the prior tasks is considered, in Def~\eqref{def:CLEMC}, both future tasks and bias due to subsequent solution are also considered. The relationship between \eqref{eq:CL_opt} and \ref{eq:clemc} is thus formalized below.
\begin{lemma}
    For $k \in [1,K]$, let $u_{k} = \{ \subv{\vect{w}}{i}, i= k,k+1, \cdots K \} $ be weight sequences from $k$ with $\mathcal{U}(k) =\{  \subv{\vect{\mathcal{W}}}{i}, i = k, k+1, \cdots \}$ -- the compact sets. Next define  \eqref{eq:CL_forget}, \eqref{eq:CL_opt} and  \eqref{eq:clemc} to write
    \begin{align}\label{eq:firstDiff}
      \subsupv{\epsilon}{*}{k+1} - \subsupv{\epsilon}{*}{k} 
       &= min_{\mathbf{k}} \; \{ \underset{\mathbf{T}_{i}}{max} \;  \{\langle \partial_{\subv{\vect{w}}{k}} \supv{V}{*}(u_{k}) , d \subv{\vect{w}}{k} \rangle  \nonumber
        + \sum_{T \in \mathbf{T}_{k}} \langle \partial_{T} \supv{V}{*}(u_{k}) , d T \rangle \}\}
       \tag{FD}
      \end{align}
\textbf{Please see Appendix B for proof.}
\label{lem:Der_capacity}
\end{lemma}

\textit{If each subsequent task is different than the previous task, the cumulative change in tasks, $d \vect{T}(k),$ is going to lead to deteriorating capacity. In particular, the change in $d \vect{T}(k),$ is going to drive a change in weights, $d \subv{\vect{w}}{k},$ which in turn drives a change in capacity. This interplay accumulates with increasing tasks eventually deteriorating capacity.}
\section{Analysis}
In this section, we perform a two-fold analysis to prove our main idea, ``capacity diverges if tasks change constantly". First, we formally prove this result. Later, we demonstrate experimentally, that the capacity diverges irrespective of the model architecture or the data used. \emph{An experimentally inclined reader can safely skip the theoretical analysis and get the same insights from our empirical observations.} We recommend reading this section to understand why capacity diverges.
\subsection{Theoretical Analysis}
We begin by deriving a lower bound on the first difference of $ \subsupv{\epsilon}{*}{k}$~(derived in Lemma \ref{lem:Der_capacity}) and then analyze the impact of the independent terms of the bound on the effective capacity.

\begin{thm} The first difference in CLEMC \eqref{eq:firstDiff} is lower bounded as \vspace{-2mm}
\begin{align} \label{eq:capacity_lb}\tag{LB}
&\subsupv{\epsilon}{*}{k} -\subsupv{\epsilon}{*}{k+1} \geq \underset{k \in \mathbf{k}}{max}  \{ \underset{\mathbf{T}_{i}}{min}   \{ \| \partial_{w_k}  J_{w^*_k}( \subv{ \mathbf{T}}{i}) \|  \|d w^*_k \| \nonumber
    +\sum_{T(k) \in \mathbf{T}_i}  \sum_{i=k}^{K} \|\partial_{T(k)} E_{t \in T(i)} \ell_{w^*_i}(t) \| \|d T(k)\| \} \}, \nonumber 
    \end{align}
    \textbf{Please see Appendix C for proof.}
\label{thm:lower_bound_capacity}
\end{thm}
It is straightforward to see that this lower bound in Theorem~\ref{thm:lower_bound_capacity} is zero, given no change in tasks ($d \vect{T}(k)$) or the weights ($d \subsupv{\vect{w}}{*}{k}$). However, in practice each time a task $k$ is introduced to the CL problem, there is a change in the value function. This change is an accumulation of the impact of the new task $k$, on all the prior tasks that in the interval $[1,k]$  ($\sum_{i = 1}^{k}$ at the outer of the two terms in \eqref{eq:capacity_lb} accumulates this change). For each task $i$ in this sum, \eqref{eq:capacity_lb} is a function of two key terms,  (I)~``the norm of the gradient of the value function with respect to the solution of the CL problem at $i^{th}$ task" and (II)~``the norm of the change in the value function due to change in the data at the $i^{th}$ task."  We now study the effects of each of these terms below.

\noindent \textbf{(I)-Capacity diverges~(deteriorates) for bounded weight updates:} To illustrate the effect of weight update, we assume that experience replay~(ER)-driven CL methods define either (i) a forgetting cost using all the available tasks, and/or (ii) utilize a regularizer on top of the forgetting cost~\cite{cumminsComparativeStudyContinual2023}. We further assume that, at each task $k$ the weights are updated for a total of $I$ steps. Under these assumptions, we show that for both settings (i) and (ii) above, the effective capacity diverges.
\begin{thm} 
\label{thm:task_nonstationary_weights}
 Fix $k \in \mathbb{N}$ and $I$, the number of weight updates required to obtain the optimal value. Assume that $\| \partial_{w_k}  J_{w^*_k}( \subv{ \mathbf{T}}{i}) \| \geq \Phi_{w}$, $\|  \partial_{T(k)} E_{t \in T(i)} \ell_{w^*_i}(t) \|\geq \Phi_{T}$, and let the smallest value of $min_{T(k)} \|d T(k)\| \geq \Phi_{dT}.$ Let $L, \mathcal{R}$ be the Lipschitz constants for the cost function and the regularization function respectively with $\alpha_{\text{MIN}}$ being the smallest learning rate. Then, $\sum_{k}^{K} d\subsupv{\epsilon}{*}{k}$ diverges as a function of $K$, and $I$ with and without the regularization factor.
 
 \textbf{Please see Appendix D for proof.}
\end{thm}
Theorem~\ref{thm:task_nonstationary_weights} demonstrates an important and novel result in the CL literature. In essence, for any CL algorithm in the literature with standard gradient driven optimization regime, capacity will diverge as long as the each subsequent tasks keeps accumulating constant albeit small differences.  Therefore, CL algorithms have the potential to result in a model that does not represent all the tasks reasonably. Moreover, this behavior is uncontrollable as tasks are unknown. 

\noindent \textbf{(II)-Capacity diverges~(deteriorates) when you have a constant change in the tasks:} To demonstrate the effect of tasks on capacity, we state the following theorem
\begin{thm}
  Under the condition of Theorem \ref{thm:task_nonstationary_weights}, let the maximum change in subsequent tasks and weights be given by $ \underset{k \in \mathbf{k}}{max} \; \{ \Phi_{T}  \Phi_{dT} \} = c.$
  Then, the $\sum_{k}^{K} d\subsupv{\epsilon}{*}{k}$ diverges as a function of $K$, and $I$ without any assumptions on the weight updates. 
  
  \textbf{Please see Appendix E for proof.}
\label{thm:task_nonstationary_tasks}
\end{thm}

Theorem~\ref{thm:task_nonstationary_tasks} shows that when a constant change is introduced into the tasks even without any assumptions on the weights, the model becomes unsuitable to represent the tasks. The impact of task similarity on CL has also been studied in  \cite{lin2023theory, evron2023continual, kim2023learnability, goldfarb2024the}. In contrast with Theorem~\ref{thm:task_nonstationary_tasks}, \cite{lin2023theory, evron2023continual, goldfarb2024the} study the impact for a linear classifier. In particular, \cite{goldfarb2024the}[Theorem 3] shows a monotonic decrease in forgetting cost as a function of similarity. For a two task case, Theorem~\ref{thm:task_nonstationary_tasks} indicates the same result in \cite{goldfarb2024the}[Theorem 3] as similar tasks will result in no change in capacity. At first, Theorem~\ref{thm:task_nonstationary_tasks} might appear contradictory to \cite{kim2023learnability}[Theorem 3.7], however, our result actually aligns with \cite{kim2023learnability}[Theorem 3.7]. Note that in the case when the overlap between distributions will keep decreasing, the loss function will proportionately increase and the risk gap will diverge. 


\begin{wrapfigure}[22]{r}{0.45\textwidth}
\includegraphics[width=0.44\textwidth]{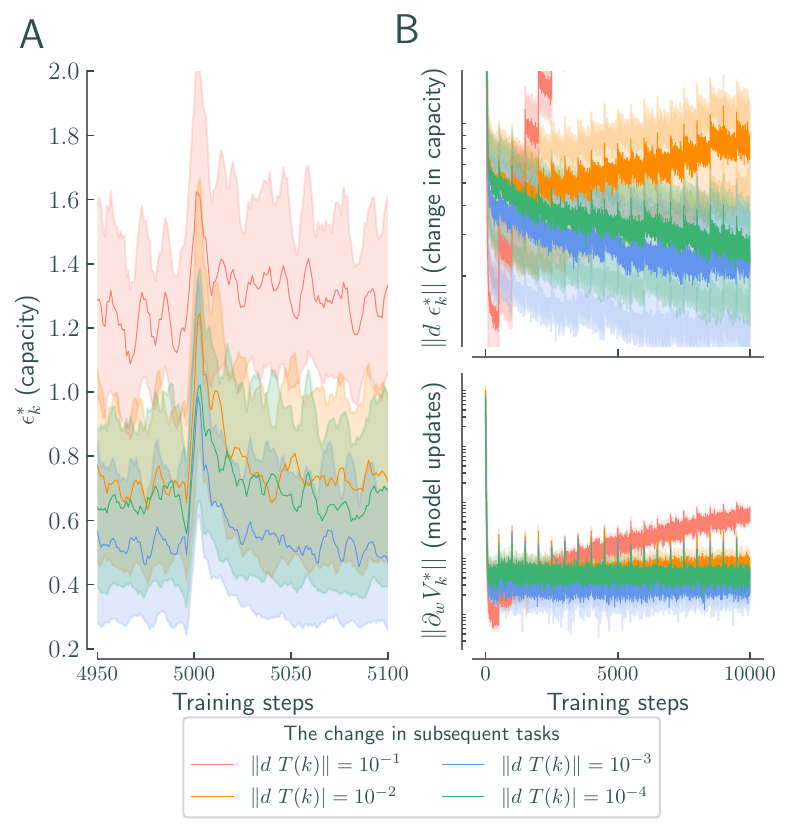}
\caption{A: Forgetting cost with ER; B: (top) capacity; (bottom) the gradient of capacity with respect to weights as a function of training steps with $\|d \epsilon^*_k \| = \|\frac{\partial V^{(*)}}{ \partial \vect{w}}\|+\|\frac{\partial V^{(*)}}{ \partial \vect{x}}\|$.}
\label{fig:sine_data}
\end{wrapfigure}
\subsection{Experimental Analysis}
In this section, we aim to substantiate the theoretical results and to that end, we develop an array of experiments where we show that capacity diverges with respect to change in tasks irrespective of the type and scale of the model. In all these experiements, we measure the capacity $\epsilon_k^*,$ the first difference in capacity
$d~\epsilon_k^*$ and the derivative of the value function with respect to weights, 
$\partial_{w} V_k^{(*)}.$ 

We reiterate that, \emph{this work does not present a new method nor does it pertain to demonstrating a new way of doing CL}, but, the goal is to elucidate how the shift in the data-distribution affects the neural network model in the CL setting. To illuminate on this perspective, we build our experiments on popular neural network architectures, namely: feed forward NN~(FNN), convolutional NN~(CNN), graph NN~(GNN) and a transformer-based model.  We argue that, for any particular model, the phenomenon of deteriorating capacity as observed on one dataset does translate to other datasets as well because the divergence of capacity is the function of how the NN model react to the shift in the data distribution.  Therefore, we choose datasets that are easier to analyze but still relevant in the CL paradigm, both in the supervised and the self-supervised learning regimes. In particular, we utilize a FNN with a synthetic sine wave dataset~\cite{javed2019meta}, a CNN with the Omniglot dataset~\cite{beaulieu2020learning}, a GNN with synthetic graph dataset and a transformer-driven large language model~(LLM) on a trillion (T) tokens dataset provided by RedPajama~\cite{together2023redpajama}. We execute FNN/CNN/GNN experiments using the JAX library and we utilize pytorch for the LLM experiments.

\noindent \textbf{Case Study 1: Feed-forward NNs} 
\emph{Setup:} For this experiment, we generate a total of twenty tasks, where each task is comprised of sine waves, generated by increasing the value of amplitude and frequency by a quantity $\left\| d \vect{T}(k)  \right\| $ to indicate distribution shift. For analysis, we observe the trend of $\subsupv{\epsilon}{*}{k}$~(capacity) for two standard methods in CL: Experience Replay (ER) shown in Fig.~\ref{fig:sine_data} and regularized ER shown in Fig.~\ref{fig:sine_data_2}. We simulate four versions of this twenty task CL problem by choosing different values of distribution shift~$\left\| d\vect{T}(k)  \right\|$, i.e. $\left\| d\vect{T}(k)  \right\|  \in \{10^{-01}, 10^{-02}, 10^{-03}, 10^{-04}\}$ and learn $20$ tasks for $10$ repetitions with mean squared error (MSE) cost and $500$ epochs per task.

\emph{Analysis of CL using ER:} In panel A of Fig. \ref{fig:sine_data}, we plot the mean of capacity,  evaluated using its upper bound through the forgetting cost evaluated using the MSE averaged across $10$ repetitions. The standard deviation is represented using a shaded region. 

We first note that, for any new task (we choose a random task at the middle of the learning process to illustrate this), there is an instantaneous increase in the capacity~(upper bounded by the forgetting cost). This increase is then minimized by the optimizer, a phenomenon known as stability gap~(\cite{harunOvercomingStabilityGap2023}) in the CL literature. We observe that, the smaller the value of $\left\| d\vect{T}(k)  \right\|,$ the closer to zero, the capacity appears to be. Our theoretical result in Theorem~\ref{thm:task_nonstationary_tasks} precisely indicates that each small change in the task leads to a proportional change in the forgetting cost and by extension, the capacity. 

We see this trend also in Fig.~\ref{fig:sine_data}, Panel B, where we plot $d \subsupv{\epsilon}{*}{k}$ with respect to training steps. For each new task, the same behavior as Fig.~\ref{fig:sine_data}, Panel A is observed. Similar to Fig.~\ref{fig:sine_data}, Panel A, the capacity of the network gets worse proportional to $\left\| d\vect{T}(k)  \right\| $~(a conclusion from Theorem~\ref{thm:task_nonstationary_tasks}). 

\begin{wrapfigure}[22]{r}{0.45\textwidth}
\centering
    \includegraphics[width=0.40\textwidth]{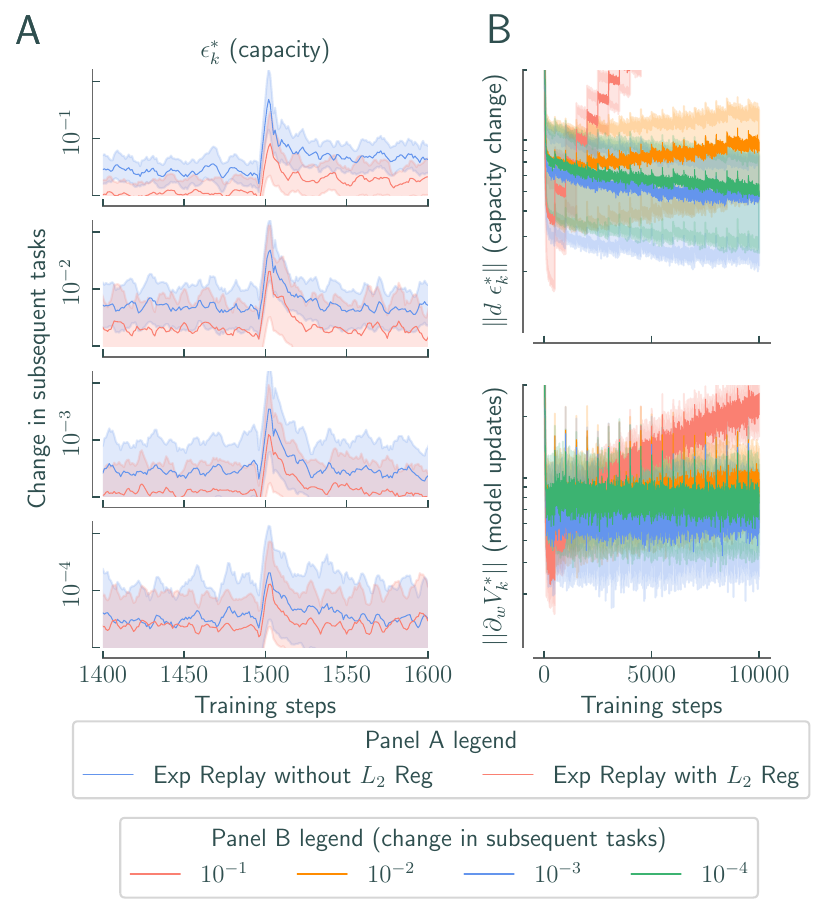}
\caption{A: Forgetting cost with ER and $L_{2}$ regularization; B: capacity; (bottom) the gradient of capacity with respect to weights as a function of training steps under $L_{2}$ regularization with $\|d \epsilon^*_k \| = \|\frac{\partial V^{(*)}}{ \partial \vect{w}}\|+\|\frac{\partial V^{(*)}}{ \partial \vect{x}}\|.$}
    \label{fig:sine_data_2}
\end{wrapfigure}
As seen in Fig.~\ref{fig:sine_data}, Panel B, using vanilla experience replay, which is supposed to compensate for the distribution shift in tasks, also exhibits deterioration in capacity. Moreover, the deterioration is proportional  to $\left\| d\vect{T}(k)  \right\| $ (green is poorer than blue, orange is worse than green, red corresponds to the worst capacity) -- an expected result as per Theorem~\ref{thm:task_nonstationary_weights}. The addition of a regularization factor does seem to improve this behavior as seen in Fig.\ref{fig:sine_data_2}, Panel B. Similarly, Fig.\ref{fig:sine_data_2}, Panel A reinforces that regularization applied to ER improves the slope of the capacity for all values of $\left\| d\vect{T}(k)  \right\|$ (the different rows in Panel A). As shown in Theorem \ref{thm:task_nonstationary_weights}, in spite of regularization, for a large enough $\left\| d\vect{T}(k)  \right\|$ capacity increases drastically (the red curve corresponding to $10^{-01}$ increases very fast) as in Fig.\ref{fig:sine_data_2}, Panel B.

\noindent \textbf{Case Study 2: Convolutional NNs}
\emph{Setup:} We now use the Omniglot dataset~\cite{cumminsComparativeStudyContinual2023, beaulieu2020learning} which is commonly used in continual~\cite{beaulieu2020learning}, and meta continual learning problems~\cite{javed2019meta} because of the presence of large numbers of tasks in contrast to the MNIST and CIFAR datasets, that are mostly image recognition datasets. We create a total of 10 classes and sequentially expose the CNN to one class at a time under the incremental class learning paradigm~\cite{lomonaco2021avalanche}.
 
 \emph{Analysis:} Overall, all the conclusions from the previous case study does carry forward. The stability gap~\cite{harunOvercomingStabilityGap2023} phenomenon is seen in Fig. \ref{fig:omniglot}, Panel A. The continuously deteriorating capacity that was observed in Fig.~\ref{fig:sine_data} for large noise values are not observed here because, there is no artificial noise being introduced here. In fact, the top plot in Fig. \ref{fig:omniglot}, Panel B shows a very stable learning behavior. However, on careful analysis, one can observe that the amount of weight updates required to attain this learning behavior keeps increasing~(bottom plot in Fig. \ref{fig:omniglot}, Panel B). This increasing requirement for larger and larger weight updates results in steady deterioration in capacity, as the model is unable to reduce the forgetting cost back to the same level for incoming tasks.

\begin{wrapfigure}[20]{r}{0.45\textwidth}
\centering
\includegraphics[width=0.40\textwidth]{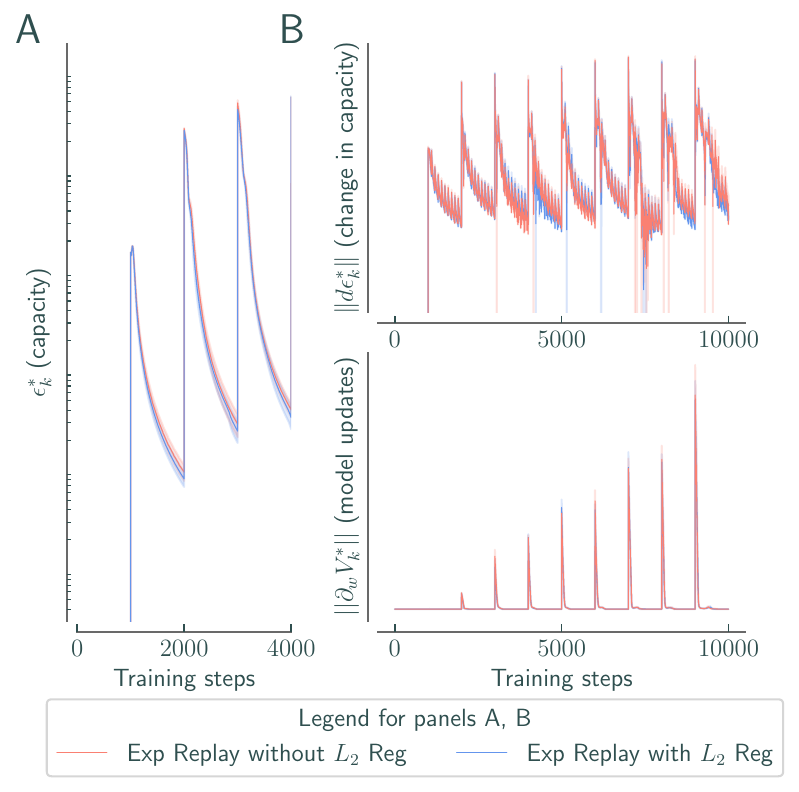}
\caption{ Panel A: Evolution in capacity with training steps. Panel B: (top)~Change in capacity and (bottom)~the gradient of capacity.}
\label{fig:omniglot}
\end{wrapfigure}

This can be observed in Fig. \ref{fig:omniglot}, Panel A where capacity at step $2000$ is better than that at step $3000$ which is better than that at step $4000$, and this is also our main contention in Theorems~\ref{thm:task_nonstationary_tasks} and \ref{thm:task_nonstationary_weights}. Although, deteriorating capacity was easier to observe in the synthetic dataset, even for a real world benchmark CL problem (with no additional noise), the theoretical results are indeed valid. 
 

\noindent \textbf{Case Study 3: Graph NN}
\emph{Setup:} We generate a total of $10$ tasks using the PyTorch geometric library~\cite{fey2019fast} with each task comprising of $4$ randomly sampled classes from a $10$-class classification problem. The key feature of this synthetic data is that both the node and edge features change. We serially feed these tasks to the GNN and train for $500$ steps each. 


 \begin{wrapfigure}[21]{r}{0.45\textwidth}
    \centering
    \includegraphics[width=0.44\textwidth]{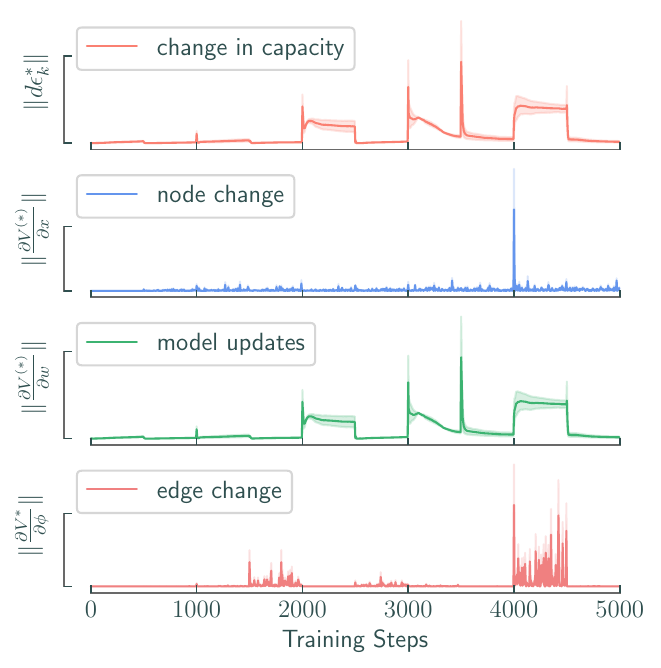}
    \caption{The evolution of capacity for graphs. The change in capacity is the sum of the node changes, model updates and edge changes.}
    \label{fig:graph}
\end{wrapfigure}
\emph{Analysis:} In this study, we again analyze the capacity deterioration from the perspective of data distribution shift due to incoming tasks which we summarize in Fig.~\ref{fig:graph}, in the top panel, we show the change in $\|d\epsilon^*_k\|$ (capacity change) for each subsequent tasks. By 
Lemma~\eqref{eq:firstDiff}, $\|d\epsilon^*_k\|$ is approximately the sum of $\|\frac{\partial V^{(*)}}{ \partial \vect{w}}\|$,  $\|\frac{\partial V^{(*)}}{ \partial \vect{x}}\|$ and $\|\frac{\partial V^{(*)}}{ \partial \vect{\phi}}\|$. To observe what introduces the change in capacity, we provide a more granular breakdown  $\|d\epsilon^*_k\|$ by contrasting it with corresponding changes in the input data. We observe that large changes in $\|d\epsilon^*_k\|$, are explained by corresponding large changes in model weights $\frac{\partial V^{(*)}}{ \partial \vect{w}}$ which is directly guided by the change in the tasks $x$ and $\phi.$  More specifically, where there is a large spike in the edge or node features~(around step $4000$), there is a large update in the weights and correspondingly in $\|d\epsilon^*_k\|$ as well. The size of the jump increases with subsequent tasks.

\noindent \textbf{Case Study 4: Transformer-based Large Language Models (8M and 134M parameters)}

\emph{Setup:} We utilize four sub-datasets (\texttt{wiki} $\rightarrow$ \texttt{git} $\rightarrow$ \texttt{arxiv} $\rightarrow$ \texttt{books}) from the RedPajama 1T tokens dataset~\cite{together2023redpajama} for both pre- and continual pre-training. We use the LLama2 tokenizer~\cite{touvron2023llama} and decoder model architecture~\cite{touvron2023llama} to construct models with 8M and 134M parameters (details in Appendix). Pre-training was done with a batch size of approx. 4M tokens for 48K steps (about 200B tokens), and a 2K-step linear warmup.  For CL, we conduct two experiments: one without ER, using data from only the current task, and another with ER, mixing 80\% current task data with 20\% from previous tasks (details on data mix in Appendix F). Each task is trained for 12K steps (about 50B tokens), starting each new task from the previous task's final checkpoint. Validation scores are computed on the \texttt{C4-en} validation set~\cite{c4} using the final checkpoint for each task. We use identical hyper-parameter settings for both models and leverage PyTorch FSDP~\cite{fsdp} on $64$ A10 (40GB) GPUs.

\vspace{-4mm}
\begin{figure*}[h]
    \centering
    \includegraphics[width=0.8\textwidth]{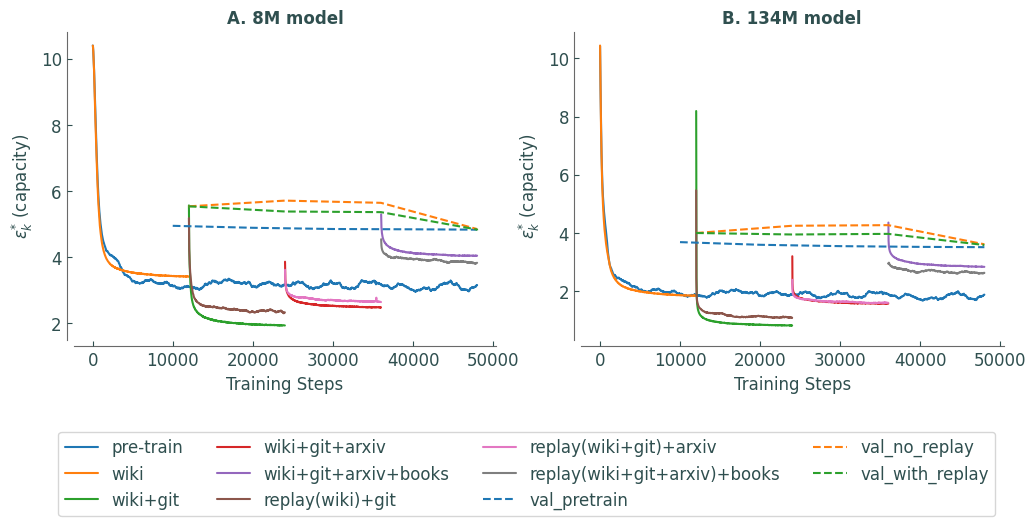}
    \caption{CL on language models demonstrate that forgetting cost increases as new tasks arrive both with and without ER. As expected, the 134M model has higher effective capacity than the 8M model.}
    \label{fig:llm_data}
\end{figure*}


\emph{Analysis:} We compare capacity (measured using its upper bound which is the forgetting cost) for continual pre-training with/ without ER of 8M and 134M parameter models in Fig.~\ref{fig:llm_data}. Pre-training capacity is shown for reference.

\textit{8M model:} Without ER, we see that the capacity initially goes down for the second task (\texttt{git}) but then keeps increasing with the arrival of each new task (\texttt{arxiv} followed by \texttt{books}). This is an expected baseline result~\cite{ramasesh2021effect} and indicates forgetting.
Even with ER, we observe an increase in capacity as new tasks arrive. This is a consequence of Theorem~\ref{thm:task_nonstationary_weights}, as the model needs to learn concepts from a mix of data from multiple tasks. The only exception occurs for the \texttt{books} task, where the cost observed with ER is lower than without ER. We attribute this to initialization bias (i.e., optimal solution from the previous task is a good initialization for the current task). This can also be inferred from Theorem \ref{thm:task_nonstationary_tasks}, where more similarity in task leads to better learning- an effect shown theoretically in \cite{lin2023theory} as well.

For reference, we add the pre-training capacity curve where all tasks are available together. Initially, the learning objectives~(both with and without ER) are relatively easier and therefore task capacity is lower than the overall pre-training capacity. However, as more tasks arrive the capacity eventually becomes higher than the pre-training capacity because the models keeps on forgetting even with ER~(Theorem \ref{thm:task_nonstationary_weights}). The validation capacity for pre-training model is always lower than both with and without ER indicating that the pre-trained models' forget less and generalize better.

\textit{134M model:} Similar to the 8M model, we observe an increase in capacity as new tasks arrive. However, owing to larger scale, the capacity values are still relatively lower than the 8M model-- an expected result, as larger models are more resilient to small changes in the tasks due to the number of parameters to help with adaptation.

\section{Discussion and Limitations}
A critical research gap in CL has been the static treatment of capacity. Existing methods define capacity in terms of fixed parameters, such as the number of neurons or data points, neglecting its dynamic evolution across tasks. However, in reality, CL exhibits a recursive dependence between weights and tasks. Our work redefines capacity as a dynamic quantity that evolves temporally, influenced by tasks, weights, and architectural adjustments. This characterization through dynamical system tools provides the mathematical constructs that will allow us to compensate for change in capacity. For instance, one could measure the value of capacity and use it as part of the optimization problem to make continual learning, ``capacity conscious''. Mathematically, this will involve adding capacity constraints into the optimization problem such that \eqref{eq:CL_opt} to achieve 
\begin{align}\label{eq:CL_cap_constraint} \tag{CC}
\supv{V}{*}(u_k) =  \underset{ u_k}{ min } \quad {\sum}_{i = k}^{K} \left[  J_{\subv{\vect{w}}{i}}( \subv{ \mathbf{T}}{i}  )  \right], \quad \text{subject to}  \quad \subsupv{\epsilon}{*}{k} -\subsupv{\epsilon}{*}{k+1} \leq \gamma.
\end{align}
In other words, we want to \textit{solve the CL problem such that the change in capacity is marginal}. This change in capacity is a function of  changes in tasks ($dT$) and weights ($d~w$), which could be measured through simple algorithmic differentiation procedure. Therefore,  our framework not only provides the avenue to study the relationship between stability gaps and loss trends in CL as described in Fig.~\ref{fig:sine_data} and Fig.~\ref{fig:omniglot} but also explains and allows the flexibility to control the learning behavior. However, the effective control of this learning behavior now depends on the constraint on capacity in \eqref{eq:CL_cap_constraint} which is dynamic and upper-bounded by $\gamma.$ 

By colloquial ML wisdom, capacity is intricately dependent on the size or architecture of the model. We showed in Theorems~\ref{thm:task_nonstationary_weights} and \ref{thm:task_nonstationary_tasks} that, with the use of traditional CL approaches, $\gamma$ could diverge. Therefore, success in CL would need the practitioner to control $\gamma$ which requires the practitioner to determine `` how should I  change the architecture to minimize the change in capacity?'' Such a construct is feasible within our approach by connecting Lemma \ref{lem:Der_capacity} and $\gamma,$ that is, bounding the right hand side of Lemma \ref{lem:Der_capacity}. Notably, no available approach in the CL literature provides the framework to enable this flexibility even though some heuristics are available~\cite{touvron2023llama}. 

In addition, task ordering dictates the dynamic evolution of the CL problem and by extension that of the capacity. Typically, in CL, the sequence of tasks are unknown apriori and one cannot replay tasks repeatedly. Therefore, the impact of task ordering can only be understood through the dependence of new task on prior tasks which is implicitly modeled as the CL problem in \eqref{eq:CL_opt} (as a sum over all past and the future tasks). In particular, the dynamic evolution of capacity at $k$ reflects ``the effect of present task on capacity and is conditioned by a fixed sequence of prior tasks.'' Naturally, a different sequence of prior tasks would result in a different evolution of capacity captured within our framework through dynamic programming principles.

Finally, within our mathematical infrastructure the contributions of inseparable components such as tasks, weights and architecture can be understood  by connecting the high-level notion of capacity evolution to low-level optimization dynamics, offering new insights into how optimization strategies (e.g., ADAM vs. SGD) impact learning outcomes. Our framework allows the possibility to analyze both transient dynamics (e.g., short-term forgetting due to large task shifts) and long-term trends (e.g., cumulative forgetting over multiple tasks). These ideas jointly enable a comprehensive understanding of the existing challenges in CL -- a precursor to efficient CL.

\section{Conclusion}
By redefining capacity as a dynamic quantity linked to tasks and weights through recursive equations, our work provides a solid theoretical foundation (with empirical validation) for understanding and addressing CL problems. Our main conclusion is that even if each subsequent task is only slightly different from the previous one, the capacity eventually diverges, rendering the model unusable.  Future research, based on this mathematical framework can explore critical questions such as the impact of specific task ordering, model scale, and optimization technqiues on forgetting dynamics and capacity evolution. Addressing these requires extensive experimentation and a shift to dynamic frameworks, as static models are insufficient for capturing the evolving nature of CL problems.

%% file: prelim.tex
    Let $x$ and $y$ be random variables corresponding to input and output probability spaces with support $\vect{\mathcal{X}}$ and $\vect{\mathcal{Y}}$ and $\mathcal{B}(\vect{\mathcal{X}})$ and $\mathcal{B}(\vect{\mathcal{Y}})$ representing the corresponding Borel algebras. Define $\vect{t}$ as a random variable denoting the joint space of $\vect{x} \times \vect{y}$ with a model $f_{(\vect{w}, \vect{h})} : \vect{\mathcal{X}} \rightarrow \vect{\mathcal{Y}}$ being specified using weights $\vect{w}$ and hyperparameters $\vect{h}$. Given compact sets $\vect{\mathcal{W}}$ over $\vect{w}$ and $\vect{\mathcal{H}}$ over $\vect{h},$ the goal is to learn the weights by searching over the hypothesis space $\vect{f} = \{ f_{(\vect{w}, \vect{h})} , \forall \vect{h} \in \vect{\mathcal{H}}, \vect{w} \in \vect{\mathcal{W}} \}$ through a loss function $\ell_{w, h}(\vect{t}).$
In this paper, we will assume that the hyperparameter/architecture is fixed and therefore, will drop the notation $h$ and denote loss simply as $\ell_{w}(t)$. Throughout the paper, we will assume $x_k = x(k)$ and use them interchangeably, and $\mathbf{k} = [1,2,\cdots, k]$. In this context, we characterize the effective model capacity as follows.

\noindent \textbf{Effective Model Capacity:} We will assume that $\ell_{w}(\vect{t})$ is continuous and twice differentiable over the support $\vect{\mathcal{X}} \times \vect{\mathcal{Y}}$ or $\vect{\mathcal{X}}$, and the compact set $\vect{\mathcal{W}}$. Under these assumptions, let $\ell_{min} = \mathcal{O}_{\mathcal{W}}(\vect{T}) =  min_{\vect{w} \in   \vect{\mathcal{W}}} \quad {E}_{ \vect{t} \in \vect{T} } [\ell_{w}( \vect{t})]$ be the optimization procedure with $\vect{T}$ being a dataset of samples $\vect{t}$ with $\vect{T} \subset \mathcal{B}(\vect{\mathcal{T}})$. Then, given the best hyperparameter/architecture configurations,  the optimization procedure $\mathcal{O}_{\mathcal{W}}$ seeks to find the weights $\vect{w}^* \in \vect{\mathcal{W}}$ that minimizes the loss over a dataset. Given this setting, we define the effective model capacity~(the upper/lower bounds derived in the appendix) as the smallest achievable loss value using $\mathcal{O}_{\mathcal{W}}$ that remains unchanged even when additional data or training is used.

\begin{defn*}\textbf{Definition 1} [Effective Model Capacity~(EMC)]
Given $\vect{\mathcal{W}}$ as the weight space and $\vect{T} \in \mathcal{B}(\vect{\mathcal{T}})$ with an optimization procedure $\mathcal{O}_{\vect{\mathcal{W}}}(\vect{T}),$ the EMC of the model $f$ is given as
\vspace{-2mm}
\begin{align*}
    \epsilon = \underset{\vect{T} \in \mathcal{B}(\vect{\mathcal{T}}) }{min}  \quad \big[ \mathcal{O}_{\vect{\mathcal{W}}}(\vect{T}) \big] =\underset{\vect{T} \in \mathcal{B}(\vect{\mathcal{T}})  }{min}  \quad \big[ \underset{ \vect{w} \in \vect{\mathcal{W}}}{ min } \quad \underset{ \vect{t} \in \vect{T} }{E} [\ell_{w}( \vect{t})] \big]
\end{align*}
\end{defn*}
\vspace{-2mm}

 Given a weight set $\subv{\vect{\mathcal{W}}}{k}$, and loss function $\subv{\ell}{\subv{\vect{w}}{k}}(t), t \in \subv{\vect{\mathcal{T}}}{k}$, the model at $k$ is denoted by $f_{\subv{\vect{w}}{k}}$, the goal of CL is to maintain memory of all observed tasks, then, the CL forgetting cost for the interval $\mathbf{k} = [1,k]$ is given as
\vspace{-2.5mm}
\begin{align*}\textstyle  \underset{ \subv{\vect{w}}{k} \in \subv{\vect{\mathcal{W}}}{k} }{ min } J_{\subv{\vect{w}}{k}}(\subv{ \mathbf{T}}{k}) =\underset{ \subv{\vect{w}}{k} \in \subv{\vect{\mathcal{W}}}{k} }{ min } \sum_{i =1}^{k} \gamma_i \left[ \underset{ \vect{t} \in \vect{T}(i)  }{E} \left[  \ell_{\subv{\vect{w}}{k}}( \vect{t}) \right] \right] \; , \; \; \forall \vect{T}(i) \in \subv{\mathbf{T}}{k},
\end{align*}
where, $\gamma$ ensures boundedness of $J_{\subv{\vect{w}}{k}}(\subv{ \mathbf{T}}{k})$ (see \cite{raghavan2021formalizing}, Lemma 1).
For a fixed $\vect{h} \in \vect{\mathcal{H}}$, the complete CL problem is
\vspace{-2mm}
\begin{align*}
\supv{V}{*}(u_k) =  \underset{ u_k}{ min } \quad {\sum}_{i = k}^{K} \left[  J_{\subv{\vect{w}}{i}}( \subv{ \mathbf{T}}{i}  )  \right],  u_{k} = \{ \subv{\vect{w}}{i}, i= k,k+1, \cdots K \}
\end{align*}
\vspace{2mm}

\noindent \textbf{CL Effective Model Capacity and Balance Point:}
For ease of exposition, we begin by stating

\begin{defn*}\textbf{Definition 2} [Forgetting Effective Model Capacity~(FEMC)]
For task $k \in [1, K]$, dataset $\mathbf{T}_k$, weight space $\mathcal{W}_k$, optimization procedure $\mathcal{O}_{\mathcal{W}_k }( \mathbf{T}_k)$, \ref{eq:EMC} at $k$, $\epsilon_k = min_{\mathbf{T}_{k}, \subv{\vect{w}}{k}} J_{\subv{\vect{w}}{k}}(\mathbf{T}_{k})$, we define FEMC at task $k$ as:
\vspace{-2mm}
\begin{align*}
      \text{FEMC}(k) = \max_{\mathbf{k}} \epsilon_{\mathbf{k}} = \max \{ \epsilon_1, \epsilon_2,  \cdots, \epsilon_k \} \end{align*}
\end{defn*}
\vspace{-2mm}

$FEMC(k)$ at each $k$ is defined by the highest forgetting loss in the interval $[1,k].$ We now define CL effective model capacity as follows.

\begin{defn*} \textbf{Definition 3} [Effective Model Capacity for CL~(CLEMC)]
For a task $k \in [1, K]$, we define CLEMC as the sum of FEMC across all possible tasks as
\vspace{-2mm}
\begin{align*}
      \subsupv{\epsilon}{*}{k} =\sum_{i =k}^{K} \text{FEMC}(i) = \sum_{i =k}^{K} \max_{\mathbf{i}}\; \epsilon_{\mathbf{i}}\end{align*}
\end{defn*}
\vspace{-2mm}

%% file: lem__FD.tex
\noindent \textbf{Lemma 1.}
    For $k \in [1,K]$, let $u_{k} = \{ \subv{\vect{w}}{i}, i= k,k+1, \cdots K \} $ be weight sequences from $k$ with $\mathcal{U}(k) =\{  \subv{\vect{\mathcal{W}}}{i}, i = k, k+1, \cdots \}$-- the compact sets. Next define  \eqref{eq:CL_forget}, \eqref{eq:CL_opt} and  \eqref{eq:clemc} to write
    \begin{align*}
      \subsupv{\epsilon}{*}{k+1} - \subsupv{\epsilon}{*}{k}
       &= min_{\mathbf{k}} \; \{ \underset{\mathbf{T}_{i}}{max} \;  \{\langle \partial_{\subv{\vect{w}}{k}} \supv{V}{*}(u_{k}) , d \subv{\vect{w}}{k} \rangle  + \sum_{T \in \mathbf{T}_{k}} \langle \partial_{T} \supv{V}{*}(u_{k}) , d T \rangle \}\}
      \end{align*}
\begin{proof}
We first derive the current forgetting cost as a function of infinitesimal change in $\supv{V}{*}(u_{k})$ in the following technical lemma.
\begin{lemma*}
Consider $k \in [0,K]$ with the forgetting cost as in \eqref{eq:CL_forget} and CL problem in \eqref{eq:CL_opt}. Then,
\begin{align}\label{eq:diff}
 &- \underset{ \subv{\vect{w}}{k}}{ min }  \; J_{\subv{\vect{w}}{k}}( \subv{ \mathbf{T}}{k}  ) = \left\langle \partial_{\subv{\vect{w}}{k}} \supv{V}{*}(u_{k}) , d \subv{\vect{w}}{k} \right\rangle  + \sum_{T \in \mathbf{T}_{k}} \left\langle \partial_{T} \supv{V}{*}(u_{k}) , d T \right\rangle  + \mathcal{O}(2)
\end{align}
where $d$ is the first difference operator, $\partial$ refers to the first derivative  and $\mathcal{O}(2)$ represent the higher order derivative terms.
\end{lemma*}
\begin{proof}
Let $u_{k} = \{ \subv{\vect{w}}{i}, i= k,k+1, \cdots K \} $ be the sequence of weights starting from $k$ with  $\mathcal{U}(k) =\{  \subv{\vect{\mathcal{W}}}{i}, i = k, k+1, \cdots \}\;$ being the sequence of their respective compact sets. Under the assumption that the optimal cost $\supv{V}{*}(u_{k})$ is given by the optimal trajectory of weights $u_{k}$ corresponding to the tasks sets $\{ \subv{\vect{T}}{i} \in \subv{\vect{\mathcal{T}}}{i}, i= k,k+1, \cdots K \}$, we can write the following system of recursive equations

\begin{subequations}
    \begin{align}
 & \;\supv{V}{*}(u_{k}) = \underset{ u_{k}  \in \mathcal{U}_{k}  }{ min } \quad {\sum}_{i = k}^{K} \left[  J_{\subv{\vect{w}}{i}}( \subv{ \mathbf{T}}{i}  )  \right] \label{eq:val_base}\\
 & \supv{V}{*}(u_{k+1}) =  \underset{ u_{k+1}  \in \mathcal{U}_{k+1}  }{ min } \; {\sum}_{i = k+1}^{K} \; \left[   J_{\subv{\vect{w}}{i}}( \subv{ \mathbf{T}}{i}  ) \right]  \label{eq:val_inductive} \\
 & \supv{V}{*}(u_{k})   = \underset{ \subv{\vect{w}}{k}}{ min }  \; J_{\subv{\vect{w}}{k}}( \subv{ \mathbf{T}}{k}  )+ \supv{V}{*}(u_{k+1}) \label{eq:val_recursive}
\end{align}
\end{subequations}
where \eqref{eq:val_base} and \eqref{eq:val_inductive} follow directly from using \eqref{eq:CL_forget} and \eqref{eq:val_recursive} is obtained by simply rewriting \eqref{eq:val_base} using \eqref{eq:val_inductive}.

Now, given two trajectories $u_{k}$ and $u_{k+1},$ the change introduced by $u_{k+1}$ to $\supv{V}{*}(u_{k})$ is given by Taylor series approximation of $\supv{V}{*}(u_{k})$ around $\subv{\vect{w}}{k}$ and  $\subv{ \mathbf{T}}{k}$ as,
 \begin{align}
 \supv{V}{*}(u_{k+1})   =  \supv{V}{*}(u_{k}) + \left\langle \partial_{\subv{\vect{w}}{k}} \supv{V}{*}(u_{k}) , d \subv{\vect{w}}{k} \right\rangle  + \sum_{T \in \mathbf{T}_{k}} \left\langle \partial_{T} \supv{V}{*}(u_{k}) , d T \right\rangle + \mathcal{O}(2) \label{eq:val_taylor} \end{align}
 where $d \subv{\vect{T}}{k}$ and
$d \subv{\vect{w}}{k}$ are the infinitesimal perturbations to data and weights respectively and $\mathcal{O}(2)$ represent higher order derivative terms. Substituting \eqref{eq:val_taylor} into \eqref{eq:val_recursive}  to get
 \begin{align*}
 \cancel{\supv{V}{*}(u_{k})}  = \underset{ \subv{\vect{w}}{k}}{ min }  \; J_{\subv{\vect{w}}{k}}( \subv{ \mathbf{T}}{k}  ) + \cancel{\supv{V}{*}(u_{k})} + \left\langle \partial_{\subv{\vect{w}}{k}} \supv{V}{*}(u_{k}) , d \subv{\vect{w}}{k} \right\rangle  + \sum_{T \in \mathbf{T}_{k}} \left\langle \partial_{T} \supv{V}{*}(u_{k}) , d T \right\rangle   + \mathcal{O}(2)    \end{align*}
 which proves the result stated in the technical Lemma.
 \end{proof}
 Using the above result, we can now prove Lemma~\eqref{lem:Der_capacity}. Towards this end, we begin by writing,
\begin{subequations}
    \begin{align}
      \subsupv{\epsilon}{*}{k} &=
 \sum_{i =k}^{K} max_{\mathbf{i}}\; \epsilon_{\mathbf{i}}  = max_{\mathbf{k}} \; \epsilon_{\mathbf{k}} + \subsupv{\epsilon}{*}{k+1} \label{eq:lemma1_recursive} \\
      \subsupv{\epsilon}{*}{k+1} - \subsupv{\epsilon}{*}{k} &= min_{\mathbf{k}} \{ \; -\epsilon_{\mathbf{k}} \}
      = min_{\mathbf{k}} \{ \; -\{ \underset{\mathbf{T}_{i}}{min} \; \underset{ \subv{\vect{w}}{i}}{ min }  \; J_{\subv{\vect{w}}{i}}( \subv{ \mathbf{T}}{i})\}, \; i \in \mathbf{k}\} \label{eq:lemma1_JF} \\
      &= min_{\mathbf{k}} \; \{ \underset{\mathbf{T}_{i}}{max} \;  (-\underset{ \subv{\vect{w}}{i}}{ min }  \; J_{\subv{\vect{w}}{i}}( \subv{ \mathbf{T}}{i})),\;  i\in \mathbf{k}\} \label{eq:lemma1_diff}
      \end{align}
\end{subequations}
where \eqref{eq:lemma1_recursive} is obtained by applying \eqref{eq:clemc}, and \eqref{eq:lemma1_JF} is obtained by rewriting $\epsilon_{\mathbf{k}}$ using \eqref{eq:CL_forget}.  Substituting \eqref{eq:diff} into \eqref{eq:lemma1_diff}, and ignoring the higher order derivative terms denoted by $\mathcal{O}(2)$~\cite{bertsekas}, we obtain the result as
\begin{subequations}
    \begin{align}
      &\subsupv{\epsilon}{*}{k+1} - \subsupv{\epsilon}{*}{k}
      = min_{k \in \mathbf{k}} \; \{ \underset{\mathbf{T}_{i}}{max} \;  (-\underset{ \subv{\vect{w}}{i}}{ min }  \; J_{\subv{\vect{w}}{i}}( \subv{ \mathbf{T}}{i}) ), i \in \mathbf{k}\} \\
       &= min_{k \in  \mathbf{k}} \; \{ \underset{\mathbf{T}_{i}}{max} \;  \{ \left\langle \partial_{\subv{\vect{w}}{k}} \supv{V}{*}(u_{k}) , d \subv{\vect{w}}{k} \right\rangle  + \sum_{T \in \mathbf{T}_{k}} \left\langle \partial_{T} \supv{V}{*}(u_{k}) , d T \right\rangle \}
      \end{align}
\end{subequations}

\end{proof}

%% file: LBFD.tex
\noindent \textbf{Theorem 1.}
The first difference in CLEMC \eqref{eq:firstDiff} is lower bounded as \vspace{-2mm}
\begin{align*}
\subsupv{\epsilon}{*}{k} -\subsupv{\epsilon}{*}{k+1}
    &\geq  \underset{k \in \mathbf{k}}{max} \; \{ \underset{\mathbf{T}_{i}}{min} \;  \{ \| \partial_{w_k}  J_{w^*_k}( \subv{ \mathbf{T}}{i}) \|  \|d w^*_k \| \nonumber \\
    &+\sum_{T(k) \in \mathbf{T}_i}  \sum_{i=k}^{K} \|  \partial_{T(k)} E_{t \in T(i)} \ell_{w^*_i}(t) \| \|d T(k)\| \} \}, \nonumber
    \end{align*}

    \begin{proof}
From Lemma~\eqref{lem:Der_capacity} we get
\begin{subequations}
\begin{align}
&\subsupv{\epsilon}{*}{k+1} -\subsupv{\epsilon}{*}{k}=
    min_{k \in \mathbf{k}} \; \{ \underset{\mathbf{T}_{k}}{max} \;  \{ \left\langle \partial_{w_k} \supv{V}{*}(u_{k}) , d w_k \right\rangle + \sum_{T \in \mathbf{T}_{k}} \left\langle \partial_{T} \supv{V}{*}(u_{k}) , d T \right\rangle \} \} \nonumber \\
    &\leq  min_{k \in \mathbf{k}} \; \{ \underset{\mathbf{T}_{k} }{max} \;  \{ \| \partial_{w_k} \supv{V}{*}(u_{k}) \| \|d \subv{\vect{w}}{k} \| + \sum_{T \in \mathbf{T}_k} \| \partial_{T} \supv{V}{*}(u_{k}) \| \|d T \| \} \} \label{eq:th1_cauchy}
\end{align}
\end{subequations}
where \eqref{eq:th1_cauchy} is obtained using Cauchy-Schwarz inequality, $\langle a , b \rangle \leq \| a\| \| b\|$. We then bound  both the gradient norm terms in \eqref{eq:th1_cauchy} as follows.

For the first gradient norm term, we assume that the optimal cost, $\supv{V}{*}$, is given by the weight trajectory $u_k$ with $ u_k = \{ w_i, i= k, k+1, \cdots K \}$. We can then bound it through the following inequalities.
\begin{subequations}
\begin{align}
 \|\partial_{w_k}\supv{V}{*}(u_{k}) \|
 & = \; \;   \|\partial_{w_k} \min_{u_{k}} \sum_{i=k}^{K} J_{w_i}( \subv{ \mathbf{T}}{i}) \| \\
 & \leq  \; \; \|\partial_{w_k}\sum_{i=k}^{K} min_{w_i} J_{w_i}( \subv{ \mathbf{T}}{i}) \| \label{eq:th1_ineq11} \\
 & \leq \;\; \|\sum_{i=k}^{K} \partial_{w_k} min_{w_i} J_{w_i}( \subv{ \mathbf{T}}{i}) \| \label{eq:th1_ineq12} \\
 & \leq \; \; \| \partial_{w_k}  min_{w_k} J_{w_k}( \subv{ \mathbf{T}}{i}) \| \label{eq:th1_ineq13}
\end{align}
\end{subequations}
where \eqref{eq:th1_ineq11} is because the norm of the gradient, with respect to weights, at the optimal cost (due to an optimal trajectory) is always less than the norm of the gradient, with respect to the weights, at a forgetting cost corresponding to any arbitrary weight trajectory. \eqref{eq:th1_ineq12} follows from the sum rule of derivatives and \eqref{eq:th1_ineq13} is because all terms from $w_{k+1}$ onwards vanish due to lack of dependence on $w_k$.

For the second norm of the gradient term in \eqref{eq:th1_cauchy}, we again write the optimal cost $\supv{V}{*}(u_k) = \sum_{i=k}^{K}  min_{w_i} J_{w_i}( \subv{ \mathbf{T}}{i})$ such that $\subv{ \mathbf{T}}{i} = \{\vect{T}(1), \cdots \vect{T}(i)\}.$ We further observe that if the optimal cost is differentiated with respect to  $T(k)$ only the $k^{th}$ term in the inner sum will remain. We can then bound it through the following inequalities.

\begin{subequations}
\begin{align}
\|\partial_{T(k)} \supv{V}{*}(u_k)\|
& \leq \; \; \|\partial_{T(k)}  \sum_{i=k}^{K}  min_{w_i} J_{w_i}( \subv{ \mathbf{T}}{i} ) \|\\
& \leq \; \; \|\sum_{i=k}^{K}  \partial_{T(k)} min_{w_i} \sum_{p=1}^{i} E_{t \in T(p)} \ell_{w_i}(t)\| \\
& \leq \; \; \| \sum_{i=k}^{K}  \partial_{T(k)} E_{t \in T(i)} \ell_{w^*(i)}(t) \| \label{eq:th1_ineq23}
\end{align}
\end{subequations}

Then, upon substituting \eqref{eq:th1_ineq13} and \eqref{eq:th1_ineq23} into
\eqref{eq:th1_cauchy} we get,
\small
    \begin{align}    & \subsupv{\epsilon}{*}{k+1} -\subsupv{\epsilon}{*}{k}
        \leq  \underset{k \in \mathbf{k}}{min} \; \{ \underset{\mathbf{T}_{i}}{max} \;  \{ \| \partial_{w_k}  J_{w^*_k}( \subv{ \mathbf{T}}{i}) \|  \|d w^*_k \| +\sum_{T(k) \in \mathbf{T}_i}  \sum_{i=k}^{K} \|  \partial_{T(k)} E_{t \in T(i)} \ell_{w^*_i}(t) \| \|d T(k)\| \} \},
    \end{align}
where we have replaced the inner minimization problem with respect to weights by the corresponding $w^*.$ Multiplication with $-1$ provides the lower bound as
\small
    \begin{align}    & \subsupv{\epsilon}{*}{k} -\subsupv{\epsilon}{*}{k+1}
        \geq  \underset{k \in \mathbf{k}}{max} \; \{ \underset{\mathbf{T}_{i}}{min} \;  \{ \| \partial_{w_k}  J_{w^*_k}( \subv{ \mathbf{T}}{i}) \|  \|d w^*_k \| +\sum_{T(k) \in \mathbf{T}_i}  \sum_{i=k}^{K} \|  \partial_{T(k)} E_{t \in T(i)} \ell_{w^*_i}(t) \| \|d T(k)\| \} \},
    \end{align}
\end{proof}

%% file: thm_non_stati_weights.tex
\noindent \textbf{Theorem 2.}
 Fix $k \in \mathbb{N}$ and $I$, the number of weight updates required to obtain the optimal value. Assume that $\| \partial_{w_k}  J_{w^*_k}( \subv{ \mathbf{T}}{i}) \| \geq \Phi_{w}$, $\|  \partial_{T(k)} E_{t \in T(i)} \ell_{w^*_i}(t) \|\geq \Phi_{T}$, and let the smallest value of $min_{T(k)} \|d T(k)\| \geq \Phi_{dT}.$ Let $L, \mathcal{R}$ be the Lipschitz constants for the cost function and the regularization function respectively with $\alpha_{\text{MIN}}$ being the smallest learning rate. Then, $\sum_{k}^{K} d\subsupv{\epsilon}{*}{k}$ diverges as a function of $K$, and $I$ with and without the regularization factor.

    \begin{proof}
We first prove the technical Lemma below.
\begin{lemma*}\label{lem:weight_update}
    Fix $k \in \mathbb{N}$ and let the weights at any task $k$  be updated for a total of $I$ steps. Assume $T(k)$ is provided through a series of batches  such that $T(k) = \{ \subsupv{\vect{t}}{i}{k}, i =1, \cdots,  I\}$ with $\subsupv{\vect{t}}{i}{k}$ be a tensor corresponding to batch of data at the $i^{th}$ step for the $k^{th}$ task, sampled uniformly from the underlying support. For the $i^{th}$ update step of the $k^{th}$ task, let the forgetting cost be denoted by $J_{\subv{\vect{w}}{k}}(\subv{ \mathbf{T}}{k})$, gradient be denoted by $ \subsupv{\vect{g}}{i}{k}$, and learning rate by $\subsupv{\alpha}{i}{k}$. Then,
    \begin{align}
     d \subv{\vect{w}}{k}^* =  - \sumv{i=0}{I-1}  \subsupv{\alpha}{i}{k} \subsupv{\vect{g}}{i}{k}
    \end{align}
\end{lemma*}
\begin{proof}
Note now that, we abuse notation to define $d\subv{\vect{w}}{k}^* = \subv{\vect{w}}{k}^{*}  - \subsupv{\vect{w}}{0}{k} = \subsupv{\vect{w}}{I}{k}  - \subsupv{\vect{w}}{0}{k}$ assuming that the optimal point is achieved after $I$ updates~(indicated by parenthesis). Then, at any particular update step, we obtain
\begin{align}
\subsupv{\vect{w}}{i+1}{k} & =   \subsupv{\vect{w}}{i}{k} -  \subsupv{\alpha}{i}{k} \subsupv{\vect{g}}{i}{k}
\end{align}
where $\subsupv{\vect{g}}{i}{k} $ is the update gradient at the this step.
\begin{align}
\subsupv{\vect{w}}{i+1}{k} & =   \subsupv{\vect{w}}{i}{k} -\subsupv{\alpha}{i}{k} \subsupv{\vect{g}}{i}{k}
\end{align}
We may now write the sum over the I steps at a
\begin{subequations}
    \begin{align}
        \subsupv{\vect{w}}{1}{k} & =   \subsupv{\vect{w}}{0}{k} -  \subsupv{\alpha}{0}{k} \subsupv{\vect{g}}{0}{k}  \\
        \subsupv{\vect{w}}{2}{k} & =   \subsupv{\vect{w}}{1}{k} -\subsupv{\alpha}{1}{k} \subsupv{\vect{g}}{1}{k} \\
        & \vdots \\
        \subsupv{\vect{w}}{I}{k} & =   \subsupv{\vect{w}}{I-1}{k} -  \subsupv{\alpha}{I-1}{k} \subsupv{\vect{g}}{I-1}{k}
    \end{align}
\end{subequations}
Adding all these terms to write
\begin{align}
d\subv{\vect{w}}{k}^* &=  -\sumv{i=0}{I-1}  \subsupv{\alpha}{i}{k} \subsupv{\vect{g}}{i}{k}
\end{align}
\end{proof}
    Given the first difference in capacity from the technical Lemma above,
    and under the assumption that $\| \partial_{w_k}  J_{w^*_k}( \subv{ \mathbf{T}}{i}) \| \geq \Phi_{w}$ and $\|  \partial_{T(k)} E_{t \in T(i)} \ell_{w^*_i}(t) \|\geq \Phi_{T}$
    \begin{subequations}
        \begin{align}
            \subsupv{\epsilon}{*}{k} - \subsupv{\epsilon}{*}{k+1}
            &\geq \underset{k \in \mathbf{k}}{max} \; \{ \underset{\mathbf{T}_{i}}{min} \;  \{ \| \partial_{w_k}  J_{w^*_k}( \subv{ \mathbf{T}}{i}) \|  \|d w^*_k \| +\sum_{T(k) \in \mathbf{T}_i}  \sum_{i=k}^{K} \|  \partial_{T(k)} E_{t \in T(i)} \ell_{w^*_i}(t) \| \|d T(k)\| \} \} \nonumber  \\
            &\geq \underset{k \in \mathbf{k}}{max} \; \{ \underset{\mathbf{T}_{i}}{min} \;  \{  \Phi_{w}\|d w^*_k \| +\sum_{T(k) \in \mathbf{T}_i}  \sum_{i=k}^{K} \Phi_{T} \|d T(k)\| \} \} \\
            &\geq \underset{k \in \mathbf{k}}{max} \; \{  \Phi_{w}\|d w^*_k \| + \underset{\mathbf{T}_{i}}{min} \; \sum_{T(k) \in \mathbf{T}_i}  \sum_{i=k}^{K} \Phi_{T} \|d T(k)\| \} \} \\
            &\geq \underset{k \in \mathbf{k}}{max} \; \{  \Phi_{w}\|d w^*_k \| + \; \sum_{T(k) \in \mathbf{T}_i}  \sum_{i=k}^{K} \Phi_{T}  \underset{T(k)}{min} \|d T(k)\| \} \}
        \end{align}
    \end{subequations}
Let the smallest value of $\underset{T(k)}{min} \|d T(k)\| \geq \Phi_{dT},$ then, we can write
        \begin{align}
           \subsupv{\epsilon}{*}{k} - \subsupv{\epsilon}{*}{k+1}
            & \geq \underset{k \in \mathbf{k}}{max} \; \{  \Phi_{w}\|d w^*_k \| + \; \sum_{T(k) \in \mathbf{T}_i}  \sum_{i=k}^{K} \Phi_{T}  \Phi_{dT} \} \\ &\geq \underset{k \in \mathbf{k}}{max} \; \{  \Phi_{w}\|d w^*_k \| \} + \; \underset{k \in \mathbf{k}}{max} \; \{ \sum_{T(k) \in \mathbf{T}_i}  \sum_{i=k}^{K} \Phi_{T}  \Phi_{dT} \}
        \end{align}
    Taking sum from $k$ to K provides with the fact that each $\mathbf{T}_k$ has a total of $k$ sub datasets.
    \begin{align}\small\subsupv{\epsilon}{*}{k} - \subsupv{\epsilon}{*}{K}
            &\geq \sum_{k}^{K} \Bigg[ \underset{k \in \mathbf{k}}{max} \; \{  \Phi_{w}\|d w^*_k \| \} + \; \underset{k \in \mathbf{k}}{max} \; \{ \sum_{T(k) \in \mathbf{T}_k}  \sum_{i=k}^{K} \Phi_{T}  \Phi_{dT} \} \Bigg]\\
            &\geq \sum_{k}^{K} \underset{k \in \mathbf{k}}{max} \; \{  \Phi_{w}\|d w^*_k \| \} + \; \sum_{k}^{K}  \underset{k \in \mathbf{k}}{max} \; \{ \sum_{T(k) \in \mathbf{T}_k}  (K-k) \Phi_{T}  \Phi_{dT} \} \\
            &\geq \sum_{k}^{K} \underset{k \in \mathbf{k}}{max} \; \{  \Phi_{w}\|d w^*_k \| \} + k(K-k)^2 \;   \underset{k \in \mathbf{k}}{max} \; \{ \Phi_{T}  \Phi_{dT} \}
\end{align}
Since, $ \underset{k \in \mathbf{k}}{max} \; \{ \Phi_{T}  \Phi_{dT} \} =   \Phi_{T}  \Phi_{dT}$, $ \underset{k \in \mathbf{k}}{max} \; \{ \Phi_{w}  \Phi_{dw} \} =   \Phi_{w}  \Phi_{dw}$ and $\|d w^*_k \| \geq \Phi_{dw},$ we write
   \begin{subequations}
        \begin{align}
            \subsupv{\epsilon}{*}{k} - \subsupv{\epsilon}{*}{K}
            \geq \sum_{k}^{K} \underset{k \in \mathbf{k}}{max} \; \{  \Phi_{w} \Phi_{dw} \} + k(K-k)^2 \;    \Phi_{T}  \Phi_{dT}\geq \sum_{k}^{K}   \Phi_{w} \Phi_{dw} + k(K-k)^2 \;    \Phi_{T}  \Phi_{dT}
        \end{align}
    \end{subequations}
We will now assume that the changes introduced by the task are bounded over all future and past tasks.
Given that $K>0,k>0, \Phi_{w}>0, \Phi_{T}>0,  \Phi_{dT}>c,$ we obtain
\begin{align}
    \subsupv{\epsilon}{*}{k} - \subsupv{\epsilon}{*}{K}
    &\geq  (K-k)\Phi_{w} \Phi_{dw} + k(K-k)^2 \;    \Phi_{T} c
\end{align}
Now, by assumption that, for each task, the optimal value of weight is obtained after updating the weights for a total of $I$ steps provides
  $ \Phi_{dw} \geq  -\sumv{i=0}{I-1}  \subsupv{\alpha}{i}{k} \subsupv{\vect{g}}{i}{k}  \geq  -\sumv{i=0}{I-1}  \subsupv{\alpha}{i}{k}(-L) \geq  \sumv{i=0}{I-1}  \alpha_{\text{MIN}} L\geq  I \alpha_{\text{MIN}} L.$ Thus, we obtain
\begin{align}
    \subsupv{\epsilon}{*}{k} - \subsupv{\epsilon}{*}{K}
    &\geq  (K-k)\Phi_{w} I \alpha_{\text{MIN}} L + k(K-k)^2 \;    \Phi_{T} c
\end{align}
Then $\subsupv{\epsilon}{*}{k} - \subsupv{\epsilon}{*}{K} $ diverges as a function of $K, k, I, c.$

Similarly, for the case with regularization we may write
  $ d\Phi_{dw} \geq  -\sumv{i=0}{I}  \subsupv{\alpha}{i}{k} \subsupv{\vect{g}}{i}{k}  \geq  -\sumv{i=0}{I-1}  \subsupv{\alpha}{i}{k} -(L+\beta\mathcal{R}) \geq  \sumv{i=0}{I-1}  \alpha_{\text{MIN}} (L+\beta\mathcal{R})\geq  I \alpha_{\text{MIN}} (L+\beta\mathcal{R}),$ where $L, \mathcal{R}$ are the Lipschitz bounds on the gradients and regularizer function respectively and $\beta>0$ is a coefficient. Thus, we obtain
  \begin{align}
    \subsupv{\epsilon}{*}{k} - \subsupv{\epsilon}{*}{K}
    &\geq  (K-k)\Phi_{w} I \alpha_{\text{MIN}} (L+\beta\mathcal{R}) + k(K-k)^2 \;    \Phi_{T} c
  \end{align}
and we observe divergence as a function of $K, k.$
\end{proof}

%% file: thm_non_stati_tasks.tex
\noindent \textbf{Theorem 3.}
  Under the condition of Theorem \ref{thm:task_nonstationary_weights}, let the maximum change in subsequent tasks and weights be given by $ \underset{k \in \mathbf{k}}{max} \; \{ \Phi_{T}  \Phi_{dT} \} = c.$
  Then, the $\sum_{k}^{K} d\subsupv{\epsilon}{*}{k}$ diverges as a function of $K$, and $I$ without any assumptions on the weight updates.

    \begin{proof}
Given the first difference in capacity,
and under the assumption that $\| \partial_{w_k}  J_{w^*_k}( \subv{ \mathbf{T}}{i}) \| \geq \Phi_{w}$ and $\|  \partial_{T(k)} E_{t \in T(i)} \ell_{w^*_i}(t) \|\geq \Phi_{T}$
    \begin{subequations}
        \begin{align}
            \subsupv{\epsilon}{*}{k} - \subsupv{\epsilon}{*}{k+1}
            &\geq \underset{k \in \mathbf{k}}{max} \; \{ \underset{\mathbf{T}_{i}}{min} \;  \{ \| \partial_{w_k}  J_{w^*_k}( \subv{ \mathbf{T}}{i}) \|  \|d w^*_k \| +\sum_{T(k) \in \mathbf{T}_i}  \sum_{i=k}^{K} \|  \partial_{T(k)} E_{t \in T(i)} \ell_{w^*_i}(t) \| \|d T(k)\| \} \} \nonumber  \\
            &\geq \underset{k \in \mathbf{k}}{max} \; \{ \underset{\mathbf{T}_{i}}{min} \;  \{  \Phi_{w}\|d w^*_k \| +\sum_{T(k) \in \mathbf{T}_i}  \sum_{i=k}^{K} \Phi_{T} \|d T(k)\| \} \} \\
            &\geq \underset{k \in \mathbf{k}}{max} \; \{  \Phi_{w}\|d w^*_k \| + \underset{\mathbf{T}_{i}}{min} \; \sum_{T(k) \in \mathbf{T}_i}  \sum_{i=k}^{K} \Phi_{T} \|d T(k)\| \} \} \\
            &\geq \underset{k \in \mathbf{k}}{max} \; \{  \Phi_{w}\|d w^*_k \| + \; \sum_{T(k) \in \mathbf{T}_i}  \sum_{i=k}^{K} \Phi_{T}  \underset{T(k)}{min} \|d T(k)\| \} \}
        \end{align}
    \end{subequations}
Let the smallest value of $\underset{T(k)}{min} \|d T(k)\| \geq \Phi_{dT},$ then, we can write \begin{align}\tiny \subsupv{\epsilon}{*}{k} - \subsupv{\epsilon}{*}{k+1}
            &\geq \underset{k \in \mathbf{k}}{max} \; \{  \Phi_{w}\|d w^*_k \| + \; \sum_{T(k) \in \mathbf{T}_i}  \sum_{i=k}^{K} \Phi_{T}  \Phi_{dT} \} \nonumber \\
            &\geq \underset{k \in \mathbf{k}}{max} \; \{  \Phi_{w}\|d w^*_k \| \}
            + \underset{k \in \mathbf{k}}{max} \; \{ \sum_{T(k) \in \mathbf{T}_i}  \sum_{i=k}^{K} \Phi_{T}  \Phi_{dT} \}
        \end{align}
    Taking sum from $k$ to K provides with the fact that each $\mathbf{T}_k$ has a total of $k$ sub datasets.
   \begin{subequations}
        \begin{align}
            \subsupv{\epsilon}{*}{k} - \subsupv{\epsilon}{*}{K}
            &\geq \sum_{k}^{K} \Bigg[ \underset{k \in \mathbf{k}}{max} \; \{  \Phi_{w}\|d w^*_k \| \} + \; \underset{k \in \mathbf{k}}{max} \; \{ \sum_{T(k) \in \mathbf{T}_k}  \sum_{i=k}^{K} \Phi_{T}  \Phi_{dT} \} \Bigg]\\
            &\geq \sum_{k}^{K} \underset{k \in \mathbf{k}}{max} \; \{  \Phi_{w}\|d w^*_k \| \} + \; \sum_{k}^{K}  \underset{k \in \mathbf{k}}{max} \; \{ \sum_{T(k) \in \mathbf{T}_k}  (K-k) \Phi_{T}  \Phi_{dT} \} \\
            &\geq \sum_{k}^{K} \underset{k \in \mathbf{k}}{max} \; \{  \Phi_{w}\|d w^*_k \| \} + k(K-k)^2 \;   \underset{k \in \mathbf{k}}{max} \; \{ \Phi_{T}  \Phi_{dT} \}
        \end{align}
    \end{subequations}
Since, $ \underset{k \in \mathbf{k}}{max} \; \{ \Phi_{T}  \Phi_{dT} \} = \Phi_{T} \Phi_{dT}$, $ \underset{k \in \mathbf{k}}{max} \; \{ \Phi_{w}  \Phi_{dw} \} =   \Phi_{w}  \Phi_{dw}$ and $\|d w^*_k \| \geq \Phi_{dw},$ we write
   \begin{subequations}
        \begin{align}
            \subsupv{\epsilon}{*}{k} - \subsupv{\epsilon}{*}{K}
            &\geq \sum_{k}^{K} \underset{k \in \mathbf{k}}{max} \; \{  \Phi_{w} \Phi_{dw} \} + k(K-k)^2 \;    \Phi_{T}  \Phi_{dT}
        \end{align}
    \end{subequations}
    Assuming that the changes introduced by the task are bounded over all future and past tasks, i.e., $\Phi_{dT}>c$, we get
    \begin{subequations}
        \begin{align}
            \subsupv{\epsilon}{*}{k} - \subsupv{\epsilon}{*}{K}
            &\geq \sum_{k}^{K}   \Phi_{w} \Phi_{dw} + k(K-k)^2 \Phi_{T}\; c
        \end{align}
    \end{subequations}
Even for a constant change in task, $\subsupv{\epsilon}{*}{k} - \subsupv{\epsilon}{*}{K} $ diverges as a function of $K.$
\end{proof}

%% file: exp_details.tex
We used the following configuration of a transformer block to instantiate the 8M model.

{Embedding layer}: (32000, 128); {Attention layer}: (k, q, v, o): (128, 128); {MLP layer}: gate\_projection (128, 256), up\_projection (128, 256), down\_projection (256, 128); {Activation function}: SiLU; {Layernorm}: RMSNorm ; {Head layer}: (128, 32000); {Attention heads}: 2 ; {Layers}: 2 {Hidden size}: $128$.

We used the following configuration of a transformer block to instantiate the 134M model.

{Embedding layer}: (32000, 768); {Attention layer}: (k, q, v, o): (768, 768); {MLP layer}: gate\_projection (768, 2048), up\_projection (768, 2048), down\_projection (2048, 768); {Activation function}: SiLU; {Layernorm}: RMSNorm ; {Head layer}: (768, 32000); {Attention heads}: 12 ; {Layers}: 12 {Hidden size}: $768$.

\textbf{Pre-training Data Mix:}
\begin{itemize}
    \item \texttt{wiki}: 0.28
    \item \texttt{git}: 0.28
    \item \texttt{arxiv}: 0.16
    \item \texttt{books}: 0.28
\end{itemize}

\textbf{Experience Replay - Data Mix:}
\begin{itemize}
    \item \texttt{wiki}: 1.0
    \item \texttt{wiki}: 0.2, \texttt{git}: 0.8
    \item \texttt{wiki}: 0.1, \texttt{git}: 0.1, \texttt{arxiv}: 0.8
    \item \texttt{wiki}: 0.06, \texttt{git}: 0.07, \texttt{arxiv}: 0.07, \texttt{books}: 0.8
\end{itemize}